\documentclass[twoside,11pt]{article}

\usepackage{jmlr2e}

\usepackage{amsmath,amsfonts,amssymb}
\usepackage{mathletters} 
\usepackage{bm}

\usepackage{subcaption}

\usepackage[shortlabels]{enumitem}

\usepackage{mdframed}
\usepackage{todonotes}

\usepackage{soul}

\usepackage{algorithm}
\usepackage{algorithmic}

\usepackage[all=normal,floats]{savetrees}

\newcommand{\beq}{\begin{equation}}
\newcommand{\eeq}{\end{equation}}
\newcommand{\beqa}{\begin{eqnarray}}
\newcommand{\eeqa}{\end{eqnarray}}
\newcommand{\beqan}{\begin{eqnarray*}}
\newcommand{\eeqan}{\end{eqnarray*}}
\newcommand{\argmax}{\mathop{\mathrm{argmax}}}
\newcommand{\argmin}{\mathop{\mathrm{argmin}}}

\newcommand{\norm}[1]{\|#1\|}

\newcommand{\bX}{{\mathbb X}}

\newcommand{\blambda}{\boldsymbol\lambda}

\newtheorem{myremark}{Remark}[section]
\newtheorem{mytheorem}{Theorem}[section]
\newtheorem{mylemma}{Lemma}[section]
\newtheorem{mycorollary}{Corollary}[section]
\newtheorem{mydefinition}{Definition}[section]

\newenvironment{proofof}[1]{\par\noindent{\bf Proof of #1\ }}{\hfill\BlackBox\\[2mm]}

\jmlrheading{1}{2000}{1-48}{4/00}{10/00}{meila00a}{Audrey Durand and Odalric-Ambrym Maillard and Joelle Pineau}

\ShortHeadings{Streaming kernel regression with unknown variance}{Durand, Maillard, and Pineau}
\firstpageno{1}

\begin{document}

\title{Streaming kernel regression with provably adaptive\\mean, variance, and regularization}

\author{\name Audrey Durand \email audrey.durand.2@ulaval.ca \\
\addr Laval University, Qu\'ebec, Canada \\
\AND
\name Odalric-Ambrym Maillard \email odalric.maillard@inria.fr \\
\addr INRIA, Lille, France \\
\AND
\name Joelle Pineau \email jpineau@cs.mcgill.ca \\
\addr McGill University, Montreal, Canada
}

\editor{Kevin Murphy and Bernhard Sch{\"o}lkopf}

\maketitle

\begin{abstract}
	We consider the problem of streaming kernel regression, when the observations arrive sequentially and the goal is to 
	recover the underlying mean function, assumed to belong to an RKHS. The variance of the noise is not assumed to be known. In this context, we tackle the problem of tuning the regularization parameter adaptively at each time step, while maintaining tight confidence bounds estimates on the value of the  mean function at each point.
	To this end, we first generalize existing results for finite-dimensional linear regression with fixed regularization and known variance to the kernel setup with a regularization parameter allowed to be a measurable function of past observations.
	Then, using appropriate self-normalized inequalities we build upper and lower bound estimates for the variance, leading to Bersntein-like concentration bounds. The later is used in order to define the adaptive regularization.
	The bounds resulting from our technique are valid uniformly over all observation points and all time steps, and are compared against the literature with numerical experiments. Finally, the potential of these tools is illustrated by an application to kernelized bandits, where we revisit the Kernel UCB and Kernel Thompson Sampling procedures, 
	and show the benefits of the novel adaptive kernel tuning strategy.
\end{abstract}



\section{Introduction}

Many applications require solving an online optimization problem for an unknown, noisy, function defined over a possibly large domain space. Kernel regression methods can learn such possibly non-linear functions by sharing information  gathered across observations. These techniques are being used in many fields where they serve a variety of applications like hyperparameters optimization~\citep{Snoek2012}, active preference learning~\citep{Brochu2008}, and reinforcement learning~\citep{Marchant2014,Wilson2014}. The idea is generally to rely on kernel regression to estimate a function that can be used for decision making and selecting the next observation point.
Algorithmically speaking, standard kernel regression involves a regularization parameter that accounts for both the complexity of the unknown target function, and the variance of the noise. While most theoretical approaches rely on a fixed regularization parameter, in practice, people have often used heuristics in order to tune this parameter adaptively with time.

This however comes at the price of loosing theoretical guarantees. Indeed, in order for theoretical guarantees (based on concentration inequalities) to hold, existing approaches~\citep{Srinivas2010,Valko2013}
require the regularization parameter in the kernel regression to be a fixed quantity. Further, they assume a prior and tight knowledge of the variance of the noise,  which is unrealistic in practice. The reason for this cumbersome assumption is to adjust the regularization parameter in the kernel regression based on this deterministic quantity, as such a choice of regularization conveys a natural Bayesian interpretation~\citep{Rasmussen2006}.
Following this intuition, given an empirical estimate of the function noise based on gathered observations, one should be able to tune the regularization automatically. This is however non-trivial,
first due to the streaming nature of the data, that allows the noise to be a measurable function of the past observations,
second because concentration bounds on the empirical variance are currently unknown in such a general kernel setup,
and finally because all existing theoretical bounds require the regularization parameter to be a deterministic constant, while we require here a parameterization that explicitly depends on past observations.
The goal of this work is to provide the rigorous tools for performing an online tuning of the kernel regularization while preserving theoretical guarantees and confidence intervals in the context of streaming kernel regression with unknown noise. We thus hope to provide a sound method for adaptive tuning that is both interesting from a practical perspective and retains theoretical guarantees.

We gently start our contributions by Theorem~\ref{thm:kLSmean} that generalizes existing concentration results (such as in~\cite{Abbasi2011,Wang2014}), and is explicitly stated for a regularization parameter that may differ from the noise.
This result paves the way to an even more general result (Theorem~\ref{thm:kLSmeantuned}) that holds when the regularization is tuned online at each step. Afterwards, we introduce a streaming variance estimator (Theorem~\ref{thm:rkhs_var_bound}) that yields empirical upper- and lower-bounds on the function noise. Plugging-in the resulting estimates 
leads to empirical Bernstein-like concentration results (Corollary~\ref{cor:empBernstein}) for the kernel regression,
where we use the variance estimates in order to tune the regularization parameter.
Section~\ref{sec:bandits} presents an application to kernelized bandits, where regret bounds for Kernel UCB and Kernel Thompson Sampling procedures are derived.
Section~\ref{sec:discussion} discusses our results and compares them against other approaches.
Finally, Section~\ref{sec:expes} shows the potential of all the previously introduced results while comparing them to existing alternatives through different numerical experiments.
We postpone most of the proofs to the appendix.

\section{Kernel streaming regression with a predictable noise process}\label{sec:kernelreg}
Let us consider a sequential regression problem. 
At each time step $t\in\Nat$, a learner picks a point $x_t\in\cX$ and gets the observation
\beqan
y_t = f_\star(x_t) + \xi_t\,,
\eeqan
where $f_\star$ is an unknown function assumed to belong to some function space $\cF$,
and $\xi_t$ is a random noise. We assume the process generating the observations is \textit{predictable} in the sense that there is a filtration $\cH=(\cH_t)_{t\in\Nat}$ such that
$x_t$ is $\cH_{t-1}$-measurable and $y_t$ is $\cH_{t}$-measurable.
Such an example is given by $\cH_t = \sigma(x_1,\dots,x_{t+1},y_1,\dots,y_t)$.
In the sub-Gaussian streaming predictable model, we assume that for some non-negative constant $\sigma^2$ the following holds
\beqan
\forall t\in\Nat, \forall \gamma \in\Real,\quad
\ln \Esp\Big[ \exp (\gamma\xi_{t})\Big|\cH_{t-1}\Big]  \leq  \frac{\gamma^2 \sigma^2}{2}\,.
\eeqan
%
%
Let 	 $k:\cX\times\cX\to\Real$ be a kernel function (that is continuous, symmetric  positive definite) on a compact set $\cX$ equipped with a positive finite Borel measure,
and denote $\cK$ the corresponding RKHS. 
We first provide a result bounding the prediction error of a standard regularized kernel estimate, where the regularization is given by a fixed parameter $\lambda>0$.

\begin{mytheorem}[Streaming Kernel Least-Squares]
\label{thm:kLSmean}
	Assume we are in the sub-Gaussian streaming predictable model.
	For a parameter  $\lambda\in\Real$, let us define the posterior mean and variances 
	after observing $Y_t = (y_1,\dots,y_t)^\top \in\Real^{t\times 1}$ as
	\beqan
	\begin{cases}
	f_{\lambda,t}(x) =&k_{t}(x)^\top({\bf K}_{t}+ \lambda I_t)^{-1}Y_t\\
	s_{\lambda,t}^2(x) =& \frac{\sigma^2}{\lambda} k_{\lambda,t}(x,x)\,\, \text{with } k_{\lambda,t}(x,x)=k(x,x) - k_{t}(x)^\top({\bf K}_{t}+ \lambda I_t)^{-1}k_{t}(x)\,.	
	\end{cases}
	\eeqan	
	where $k_{t}(x) =  (k(x,x_{t'}))_{t'\leq t}$ is a $t\times 1$ (column) vector	and ${\bf K}_{t} = (k(x_{s},x_{s'}))_{s,s'\leq t}$. 
Then $\forall\delta\!\in\![0,1]$, with probability higher than $1\!-\!\delta$, it holds simultaneously over all $x\!\in\!\cX$ and $t\!\geq\! 0$,
	\beqan
|f_\star(x)\!-\!f_{\lambda,t}(x)| \!\leq \!
\sqrt{	\frac{k_{\lambda,t}(x,x)}{\lambda}}\bigg[\!\sqrt{\lambda}\norm{f_\star}_{\cK}
	\!+\!
	\sigma\sqrt{2\ln(1/\delta) + 2\gamma_t(\lambda)}	
	\bigg]\,,
	\eeqan
	where the quantity  $\quad\gamma_t(\lambda)= \frac{1}{2}
	\sum_{t'=1}^t\!\ln\!\Big(\!1\! +\! \frac{1}{\lambda}k_{\lambda,t'-1}(x_{t'},x_{t'})\Big)\quad$ is the information gain.
\end{mytheorem}
\begin{myremark}
	This result should be considered as an extension of \cite[Theorem~2]{Abbasi2011} from finite-dimensional to possibly infinite dimensional function space. It is a non-trivial result as the Laplace method must be amended in order to be applied.
\end{myremark}
\begin{myremark}
This result holds uniformly over all $x\in\cX$ and most importantly over all $t\geq 0$, thanks to a random stopping time construction (related to the occurrence of \emph{bad events}) and a self-normalized inequality handling this stopping time. This is in contrast with results such as \cite{Wang2014}, that are only stated separately for each $t$.
\end{myremark}
\begin{myremark}
	The quantity $\gamma_{t}(\lambda)$ directly generalizes the classical notion of information gain~\citep{Cover1991}, that is recovered for the choice of regularization $\lambda=\sigma^2$.	
\end{myremark}

The case when $\lambda=\lambda^\star\eqdef\sigma^2/\|f\|_\cK^2$ is of special interest, since we get on the one hand
\beqan
f_{t}^\star(x) &=& k_t(x)^\top({\bf K_t}+ \lambda^\star I_t)^{-1}Y_t\\
s_{t}^{2\star}(x) &=& \|f\|_\cK^2k_t(x,x)\,\, \text{with } k_t(x,x)=k(x,x) - k_t(x)^\top
({\bf K_t}+  \lambda^\star I_t)^{-1}k_t(x)
\eeqan	
and on the other hand 
$\|f\|_\cK k_t(x,x)^{1/2}\Big[1
+\sqrt{2\ln(1/\delta) + 2\gamma_t(\lambda^\star)}
\Big]$.
In practice however, neither $\|f\|_\cK^2$ nor $\sigma^2$ may be known exactly. In this paper, we  assume that an upper bound $C$ is given on $\|f\|_\cK$. Then, we want to build 
an estimate of $\sigma^2$ at each time $t$ in order to tune $\lambda$. Using a sequence of regularization parameters $(\lambda_t)_{t \geq 1}$ that is tuned adaptively based on the past observations requires to modify the previous theorem (it is only valid for a deterministic $\lambda$) into the following more general statement:
%


\begin{mytheorem}[Streaming  Kernel Least-Squares with online tuning]\label{thm:kLSmeantuned}	
Under the same assumption as Theorem~\ref{thm:kLSmean}, let $\blambda=(\lambda_t)_{t\geq 1}$  be a  \textbf{predictable} positive sequence of parameters, that is $\lambda_t$ is $\cH_{t-1}$-measurable for each $t$. 
Assume that for each $t$, $\lambda_t \geq \lambda_\star$ holds for a positive constant $\lambda_\star$.
Let us define the modified posterior mean and variances 
after observing $Y_t\in\Real^t$ as
\beqan
\begin{cases}
	f_{\blambda,t}(x)\!\! &= k_{t}(x)^\top({\bf K}_{t}+ \lambda_{t+1}I_t)^{-1}Y_t\\
	s_{\blambda,t}^2(x)\!\! &= \frac{\sigma^2}{\lambda_{t+1}} k_{\lambda_{t+1},t}(x,x)\,\, \text{with } k_{\lambda,t}(x,x)\!=\!k(x,x) \!-\! k_{t}(x)^\top\!
	({\bf K}_{t}+ \lambda I_t)^{-1}k_{t}(x)\,,
\end{cases}
\eeqan	
where $k_{t}(x) \!=\!  (k(x,x_{t'}))_{t'\leq t}$, and ${\bf K}_{t} \!=\! (K(x_s,x_{s'}))_{s,s'\leq t}$.
Then for all $\delta\!\in\![0,1]$, with probability higher than $1-\delta$, it holds simultaneously over all $x\in\cX$ and $t\geq0$
\beqan
|\!f_\star(x)\!-\!f_{\blambda,t}(x)| \!\leq \!
\sqrt{\frac{k_{\lambda_{t+1},t}(x,x)}{\lambda_{t+1}}}\Big[\!\sqrt{\!\lambda_{t\!+\!1}}\norm{f_\star}_{\cK}
\!+\!
\sigma\sqrt{2\ln(\!1/\delta) \!+\! 2\gamma_{t}(\lambda_\star)}
\Big]\,.
\eeqan
The proof is presented in Appendix~\ref{app:concentration}.
\end{mytheorem}
The regularization parameter $\lambda_{t+1}$ is therefore used in conjunction with previous data up to time $t$ to provide the posterior regression model (mean and variance) that is used in return to acquire the next observation $y_{t+1}$ on point $x_{t+1}$.
\begin{myremark}
	Since $\lambda_t$ is allowed to be $\cH_{t-1}$-measurable,
	this gives theoretical guarantees for virtually any adaptive tuning procedure of the regularization parameter. 
\end{myremark}

\begin{myremark}
	The assumption that $\lambda_t\geq \lambda_\star$ will be naturally satisfied for the choice of regularization we consider.
\end{myremark}
\section{Variance estimation}
\label{sec:variance_estimation}

We now focus on the estimation of the variance parameter of the noise in the case when it is unknown, or loosely known. Theorem~\ref{thm:kLSmeantuned} suggests 
to define the sequence $(\lambda_t)_{t \geq 1}$ by
%
\beqa\label{eqn:lambdat}
\lambda_t = \sigma_{+,t-1}^2/C^2\quad\text{with}\quad
\sigma_{+,t} =  \min\{\tilde \sigma_{+,t}, \sigma_{+,t-1} \}\quad\text{and}\quad
\sigma_{+,0} = \sigma_+\,,
\eeqa
%
where $\sigma_+\geq \sigma$ is an initial loose upper bound on $\sigma$ 
and $\tilde \sigma_{+,t}$ is an upper-bound estimate on $\sigma$ built from all observations gathered up to time $t$ (inclusively). This ensures that $\lambda_t$ is $\cH_{t-1}$ measurable for all $t$ and satisfies $\lambda_t\geq \lambda_\star$ with high probability, where $\lambda_\star = \sigma^2/C^2$. 
The crux is now to define the upper-bound estimate $\sigma_{+,t}$ on $\sigma$.
In order to get a variance estimate, one obviously requires more than the sub-Gaussian assumption, since the term $\sigma^2$ has no reason to be tight (the inequality remains valid when $\sigma^2$ is replaced with any larger value). In order to convey the minimality of $\sigma^2$, we assume that the noise sequence is both $\sigma$-sub-Gaussian and second-order\footnote{The term on the right-hand side corresponds to the cumulant generating function of the chi-squared distribution with 1 degree of freedom. This assumption naturally holds for Gaussian variables.} $\sigma$-sub-Gaussian, in the sense that
\beqan
\forall t, \forall \gamma < \frac{1}{2\sigma^2}
\quad
\ln\Esp\bigg[\exp(\gamma \xi_t^2) \bigg|\cH_{t-1}\bigg] \leq -\frac{1}{2}\ln\Big(1-2\gamma \sigma^2\Big)\,.
\eeqan

\vspace{-4mm}
\begin{myremark}
	To avoid any technicality, one may assume that $\xi_t|\cH_{t-1}$ is exactly $\cN(0,\sigma^2)$, in which case it is trivially second-order $\sigma$-sub-Gaussian.
\end{myremark}

Now let $\quad\hat \sigma^2_{\lambda,T} = \frac{1}{T}\sum_{t=1}^T (y_t- f_{\lambda,T}(x_t))^2\quad$ denote the (slightly biased) variance estimate for a regularization parameter $\lambda$.

\vspace{-2mm}
\begin{mytheorem}[Streaming Kernel variance estimate]
	\label{thm:rkhs_var_bound}
	Assume we  are in the predictable second-order $\sigma$-sub-Gaussian streaming regression model, with a predictable positive sequence $\blambda$ such that $\lambda_t \geq \lambda_\star$ holds for all $t$.	Let us introduce the following quantities
	\beqan
	&C_{t}(\delta)=\ln(e/\delta)\big[1+\ln(\pi^2\ln(t)/6)/\ln(1/\delta)\big],\qquad D_{\lambda,t}(\delta)= 2\ln(1/\delta) +2\gamma_t(\lambda)\\
	&\text{
		and finally}\quad \alpha = \max \Big( 1 - \sqrt{\frac{C_t(\delta')}{t}} - \sqrt{\frac{C_t(\delta') + 2 D_{\lambda_\star,t}(\delta')}{t}},0 \Big)\,.
\eeqan
	Then, let us introduce the following variance bounds, defined differently depending on whether a deterministic  upper bound $\sigma_{+} \geq \sigma$ is known (case 1)  or not (case 2).
	\beqan
	 \sigma_{+,t}(\lambda,\lambda_\star) 
	 =\begin{cases}
\hat \sigma_{\lambda,t}+ \sigma_+ \bigg( \sqrt{\frac{C_t(\delta')}{t}} + \sqrt{\frac{C_t(\delta')+2D_{\lambda_\star,t}(\delta')}{t}} \bigg) + \sqrt{\frac{2 \sigma_+ \lVert f_\star \rVert_\cK \sqrt{\lambda D_{\lambda_\star,t}(\delta')}}{t}} 
&\text{(case 1)}\\
\frac{1}{\alpha^2} \Bigg( \sqrt{\hat \sigma_{\lambda,t} \alpha+\frac{\lVert  f_\star \rVert_\cK \sqrt{\lambda D_{\lambda_\star,t}(\delta')}}{2t}} + \sqrt{\frac{\lVert  f_\star \rVert_\cK \sqrt{\lambda D_{\lambda_\star,t}(\delta')}}{2t}} \Bigg)^2&
\text{(case 2)}
\end{cases}\\
 \sigma_{-,t}(\lambda) =
 \begin{cases}
\hat \sigma_{\lambda,t} - \sigma_+ \sqrt{\frac{2C_t(\delta')}{t}} - \lVert f_\star \rVert_\cK \sqrt{\frac{\lambda}{t} \bigg( 1 - \frac{1}{\max_{t' \leq t} (1 + \frac{1}{\lambda}k_{\lambda,t'-1}(x_{t'}, x_{t'}))} \bigg)}&\text{(case 1)}\\
\bigg[ \hat \sigma_{\lambda,t} - \lVert f_\star \rVert_\cK \sqrt{\frac{\lambda}{t} \bigg(1 - \frac{1}{\max_{t' \leq t}(1 + \frac{1}{\lambda}k_{\lambda,t'-1}(x_{t'}, x_{t'}))} \bigg)} \bigg] \bigg( 1 + \sqrt{\frac{2 C_t(\delta')}{t}} \bigg)^{-1}&\text{(case 2)}.
\end{cases}
	\eeqan
		Then with probability higher than $1-3\delta'$, it holds simultaneously for all $t\geq0$
	\begin{align*}
	 \sigma_{-,t}(\lambda_t) \leq \sigma \leq \sigma_{+,t}(\lambda_t,\lambda_\star)\,.	
	\end{align*}
The proof is presented in Appendix~\ref{app:variance_estimation}.
\end{mytheorem}

\vspace{-2mm}
\begin{myremark}
	The case when absolutely no bound is known on the noise $\sigma^2$ is challenging in practice. 
	In this case, it is intuitive that one should not be able to recover the noise with too few samples.
	The bound stated in Theorem~\ref{thm:rkhs_var_bound} (see Appendix~\ref{app:variance_estimation}) supports this intuition, as when the number of observations is too small, then $\alpha=0$ and the corresponding bound becomes trivial ($\sigma\leq \infty$). 
\end{myremark}
\begin{myremark}
	In the variance bounds of Theorem~\ref{thm:kLSvar} the term $\norm{f_\star}_{\cK}$ appears systematically with the factor $\sqrt{\lambda}$. This suggests we need to choose $\lambda$ proportional to $1/\|f_\star\|_\cK^2$, which gives further justification to the target $\lambda_\star= \sigma^2/C^2$, where $C$ is a known upper bound on $\norm{f_\star}$.
\end{myremark}
\begin{myremark}
	In practice, we advice to choose the best of case 1 and case 2 bounds when $\sigma_+\geq \sigma$ is known.
\end{myremark}

In order to \textit{estimate} the upper bound $\sigma_{+,t}(\lambda, \lambda_\star)$, one needs at least a lower-bound on $\lambda_\star$. Let us define
\beqa\label{eqn:sigma_minus}
\sigma_{-,t} =  \max\{\tilde \sigma_{-,t}, \sigma_{-,t-1} \}\quad\text{with}\quad
\sigma_{-,0} = \sigma_-\,,
\eeqa
where $0 \leq \sigma_- \leq \sigma$ is a initial lower-bound on $\sigma$ and $\tilde \sigma_{-,t}$ is a lower-bound estimate on $\sigma$ built from all observations gathered up to time $t$ (inclusively). Then, one way to proceed is, at each time step $t \geq 1$, to build an estimate $\tilde \sigma_{-,t} = \sigma_{-,t}(\lambda)$, which in return can be used to compute the lower quantity $\lambda_- \leq \lambda_\star$, and obtain the estimate $\tilde \sigma_{+,t}=\sigma_{+,t}(\lambda, \lambda_-) \geq \sigma_{+,t}(\lambda, \lambda_\star)$. Then, we compute the predictable sequence $\blambda$ as described by equation~\ref{eqn:lambdat}. Further replacing the variance $\sigma$ with its estimate $\sigma_{+,t}$ using a union bound in the result of Theorem~\ref{thm:kLSmeantuned}, we derive confidence bounds that are fully computable in the context where the regularization parameter is adaptively tuned and the function noise is unknown. 
This is summarized in the following empirical Bernstein-style inequality:

\begin{mycorollary}[Kernel empirical-Bernstein inequality]\label{cor:empBernstein}
	Assume that $C\geq \|f\|_\cK$.
    Let us define the following noise lower-bound for each $t\geq 1$
    \beqan
    \sigma_{-,t} = \max \{ \sigma_{-,t}(\lambda_{t-1}), \sigma_{-,t-1} \}
    \eeqan
    and define $\lambda_{-} = \sigma^2_{-,t}/C^2$ as the corresponding lower bound on $\lambda_\star$.
	Then, let us define the following noise upper bound for each $t\geq 1$
	\beqan
	\sigma_{+,t} \!= \! \min\{ \sigma_{+,t}(\lambda_{t-1},\lambda_-), \sigma_{+,t-1} \}\,.
	\eeqan
    Define the regularization parameterizing the regression model used for acquiring observation at time $t$
    to be $\lambda_t = \sigma_{+,t}^2/C^2$, according to Equation~\ref{eqn:lambdat}.
	Then with probability higher than $1-4\delta$, the following is valid simultaneously  for all $x\in\cX$ and $t\geq0$,
    \begin{align}
        &\big|\!f_\star(x)\!-\!f_{\lambda_t,t}(x)\big| \!\leq \! \sqrt{\frac{k_{\lambda_t,t}(x,x)}{\lambda_t}}B_{\lambda_t,t}(\delta)
        \qquad \text{where} \nonumber \\
        \label{eqn:Bt}
        & B_{\lambda_t,t}(\delta) = \!\sqrt{\!\lambda_t}C \!+\! \sigma_{+,t} \sqrt{2\ln(1/\delta) + 2 \gamma_t(\lambda_-)}\,.
    \end{align}  
    %
\end{mycorollary}

\begin{myremark}
	This result is especially interesting since it provides a fully empirical confidence envelope function around $f_\star$. When an initial bound on the noise $\sigma_+$ is known and considered to be tight, one may simply choose the constant deterministic sequence $\blambda = (\lambda,\dots,\lambda)$, in which case the same result holds for $\lambda_- = \lambda$ and $\sigma_{+,t} = \sigma_+$.
\end{myremark}

We observe from Theorem~\ref{thm:rkhs_var_bound} that the tightness of the noise estimates depends on the $\lambda$ parameter that is used for computing $\tilde \sigma_{-,t}$ and $\tilde \sigma_{+,t}$. 
Since $\sigma^2/C^2 \leq \lambda_t \leq \sigma_+^2/C^2$ holds with high probability by construction, using such an adaptive 
$\lambda_t$ should yield tighter bounds than  using a fixed $\sigma_+^2/C^2$.
This is supported by the numerical experiments of  Section~\ref{sec:expes:empirical_variance_estimage}.


\section{Application to kernelized bandits}
\label{sec:bandits}

Here is a direct application of our results in the framework of stochastic multi-armed bandits with structured arms embedded in an RKHS~\citep{Srinivas2010,Valko2013}.
At each time step $t \geq 1$, a bandit algorithm recommends a point $x_t$ to sample and observes a noisy outcome $y_t = f_\star(x_t) + \xi_t$, where $\xi_t \sim \Normal(0, \sigma^2)$. Let $\star = \argmax_{x\in\bX} f_\star(x)$ be the optimal arm. The goal of an algorithm is to pick a sequence of points $(x_t)_{t \leq T}$ that minimizes the cumulative regret
\begin{align}
    \label{eqn:regret}
    \kR_T = \sum_{t=1}^T f_\star(\star) - f_\star(x_t).
\end{align}
In this context, one needs to build tight confidence sets on the mean of each arm, and this will be given by Corollary~\ref{cor:empBernstein}. We illustrate our technique on two main bandit strategies: Upper Confidence Bound (UCB)~\citep{Auer2002} and Thompson Sampling (TS)~\citep{Thompson1933}; both are adapted here to the kernel setting with unknown variance.


\begin{mydefinition}[Information gain with unknown variance]
We define the information gain at time $t$  for a regularization parameter $\lambda$ to be
\beqan
\gamma_t(\lambda)= \frac{1}{2}\sum_{t'=1}^t \ln\Big(1+ \frac{1}{\lambda}k_{\lambda,t'-1}(x_{t'},x_{t'})\Big)\,.
\eeqan
\end{mydefinition}
This definition directly extends the usual definition of information gain, that can be recovered by choosing $\lambda=\sigma^2$. The following extension of Lemma~7 in \cite{Wang2014} (see also \cite{Srinivas2012}) to the case when the variance is estimated plays an important role in the regret analysis of both algorithms.
\begin{mylemma}[From sum of variances to information gain]
\label{lem:infogain}
	Let us assume that the kernel is bounded by $1$ in the sense that $\sup_{x\in\cX}k(x,x)\leq 1$.
	Let $\blambda$ be any sequence such that $\forall \lambda \in \bslambda,  \lambda \geq \sigma^2 / C^2$. For instance, this is satisfied with high probability when using
	Equation~\ref{eqn:lambdat}. Then, it holds
	\beqan
	\sum_{t=1}^T s_{\blambda,t-1}^2(x_t) = \sigma^2 \sum_{t=1}^T \frac{1}{\lambda_t}k_{\lambda_t,t-1}(x_t,x_t) \leq \frac{2 C^2}{\ln(1 + C^2/\sigma^2)} \gamma_T(\sigma^2/C^2)\,.
	\eeqan
\end{mylemma}
In the sequel, it is useful to bound  the confidence bound term
$B_{\lambda_t,t}(\delta)$ from Equation~\ref{eqn:Bt}.
\begin{mylemma}[Deterministic bound on the confidence bound]\label{lem:Bt}
	Assume that we are given a constant $0<\sigma_{-}< \sigma$, so that 
$\sigma_{t,-} \geq \sigma_{-}$ holds for all $t$. Then for all $t\leq T$,
the confidence bound term is upper-bounded by the following deterministic quantity
\begin{align*}
B_{\lambda_t,t}(\delta) 	&\leq \sigma_+\bigg(1+ \sqrt{2\ln(1/\delta) + 2 \gamma_T(\sigma_{-}^2/C^2)}\bigg)\,.
\end{align*}
Further, we have $\gamma_t(\sigma_{t,-}^2/C^2) = \gamma_t(\sigma_{-}^2/C^2)+O(1/\sqrt{t})$.
\end{mylemma}
\begin{myremark} The term $\sigma_+$ can be replaced with
	a more refined term $\sigma + O(1/\sqrt{t})$ thanks to the confidence bounds on the variance estimates.
\end{myremark}

\paragraph{Kernel UCB with unknown variance}

The upper bound on the error can be used directly in order to build a UCB-style algorithm.
Formally, the vanilla UCB algorithm~\citep{Auer2002} corresponding to our setting picks at time $t$ the arm
\begin{align}
    \label{eqn:kernel_ucb}
    x_t \in \argmax_{x\in\cX} f^+_{\lambda_t,t-1}(x)\quad \text{ where } \quad
    f^+_{\lambda,t}(x)=f_{\lambda,t}(x) + \sqrt{\frac{k_{\lambda,t}(x,x)}{\lambda}} B_{\lambda,t}(\delta)\,.
\end{align}
Following the regret proof strategy of \cite{Abbasi2011}, with some minor modifications, yields the following guarantee on the regret of this strategy:
\begin{mytheorem}[Kernel UCB with unknown noise and adaptive regularization]\label{thm:KernelUCB}
	With probability higher than $1-\delta$, the regret of Kernel UCB with adaptive regularization and variance estimation satisfies for all $T\geq0$ (recall that
 $B_{\lambda_{t+1},t}(\delta)$ is defined in Equation~\ref{eqn:Bt}):
	\beqan
	\kR_T \leq	 2\sum_{t=1}^T\sqrt{\frac{k_{\lambda_t,t-1}(x_t,x_t)}{\lambda_t}} B_{\lambda_t,t-1}(\delta/4)\,.
	 \eeqan
	 In particular, we have
	 \beqan
	 \kR_T &\leq& 2\frac{\sigma_+}{\sigma}\Big(1 + \sqrt{2\ln(4/\delta) + 2 \gamma_T(\sigma_{-}^2/C^2)}\big)C
	 \sqrt{T \frac{2\gamma_T(\sigma^2/C^2) }{\ln(1 + C^2/\sigma^2)} }\,.
	 \eeqan
\end{mytheorem}
\begin{myremark}
	This result that holds simultaneously over all time horizon $T$ extends that of \cite{Abbasi2011} first to kernel regression and then to the case when the variance of the noise is unknown. This should also be compared to \cite{Valko2013} that assumes bounded observations, which implies a bounded noise (with known bound) and a bounded $f_\star$, and \cite{Srinivas2010} that provides looser bounds. 
\end{myremark}


\paragraph{Kernel TS with unknown variance}

Another application of our confidence bounds is in the analysis of Thompson sampling in the kernel scenario. Before presenting the result, let us say a few words about the design of TS algorithm in a kernel setting. Such an algorithm requires sampling from a posterior distribution over the arms. It is natural to consider a Gaussian posterior with posterior means and variances given by the kernel estimates. However, it has been noted in a series of papers~\citep{Agrawal2014,Abeille2016} that, in order to obtain provable regret minimization guarantees, the posterior variance should be inflated (although in practice, the vanilla version without inflation may work better). Following these lines of research, and owing to our novel confidence bounds, we derive the following TS algorithm using a posterior variance inflation factor $v_t^2$.

\begin{algorithm}
	Parameters: regularization sequence $\blambda$, variance inflation factor $v_t^2$ for each $t$.
    
	\begin{algorithmic}[1]
		\FORALL{$t \geq 1$}
		\STATE compute the posterior mean $\hat {\bf f}_{t-1}=(f_{\lambda_t,t-1}(x))_{x\in\bX}$
		\STATE compute the posterior covariance $\hat {\bm\Sigma}_{t-1}= \frac{\sigma_{+,t-1}^2}{\lambda_t} \big( k_{\lambda_t,t-1}(x,x')\big)_{x,x'\in\bX}$
		\STATE sample $\tilde f_t = \Normal(\hat {\bf f}_{t-1}, v_t^2 \hat {\bm\Sigma}_{t-1})$
		\STATE play $x_t = \argmax_{x \in \bX} \tilde f_t(x)$
		\STATE observe outcome $y_t = f_\star(x_t) + \xi_t$
		\ENDFOR
	\end{algorithmic}
	\caption{{Kernel TS} with adaptive variance estimation and regularization tuning}
	\label{alg:kernel_ts}
\end{algorithm}
\begin{myremark} 
	The algorithm does not know the variance $\sigma^2$ of the noise, but uses an upper estimate
	$\sigma_{+,t-1}^2$.
	\end{myremark}

\begin{myremark}
We assume that the set of arms $\bX$ is \emph{discrete}. This is merely for practical reasons since otherwise updating the estimate of $f_\star$ in a RKHS requires memory and computational times that are unbounded with $t$. This also simplifies the analysis.
\end{myremark}

The following regret bound can then be obtained after some careful but easy adaptation of \cite{Agrawal2014}. 
We provide the proof of this result in Appendix~\ref{app:bandits}, which can be of independent interest, being a more rigorous and somewhat simpler rewriting of the original proof technique from \cite{Agrawal2014}.

\begin{mytheorem}[Regularized Kernel TS with variance estimate]
	\label{thm:kernel_ts}
	Assume that the maximal instantaneous pseudo-regret  $R=\max_{x\in\bX}\big(f_\star(\star)-f_\star(x)\big)$ is finite. Then, the regret of Kernel TS (Algorithm~\ref{alg:kernel_ts}) with $v_t = \frac{B_{\lambda_t,t-1}(\delta/4)}{\sigma_{+,t-1}}$ after $T$ episodes is $\cO(C \sqrt{T \ln(T |\bX|)} \gamma_T(\sigma^2/C^2))$ with probability $1 - 3 \delta$. More precisely,	with probability  $1-3\delta$, 	the regret is bounded for all $T\geq0$:
	\beqan
	\kR_T
	&\leq& C_{1,T}\bigg(\sum_{t=1}^T \sqrt{\frac{k_{\lambda_t,t-1}(x_t,x_t)}{\lambda_t}} B_{\lambda_t,t-1}(\delta/4)\bigg)+C_2 R\sqrt{T \ln(1/\delta)}+   4\pi e R \delta\,, 
	\eeqan
	where $C_{1,T} = (4\sqrt{\pi e}+1)\bigg(1+\sqrt{2\ln\Big(\frac{T(T+1) |\bX|}{\sqrt{\pi}\delta}\Big)}\bigg)$ and $C_2=\sqrt{ 8 \pi e(1 + \delta\sqrt{4\pi e})^2}$.
    
    Further, we have
	\beqan
	\kR_T &\leq& C_{1,T}\frac{\sigma_+}{\sigma}\Big( 1 +\sqrt{2\ln(4/\delta) + 2 \gamma_T(\sigma_-^2/C^2)} \Big)  C\sqrt{T \frac{2 \gamma_T(\sigma^2/C^2)}{\ln(1 + C^2/\sigma^2)}} \\
	&&+C_2 R\sqrt{T \ln(1/\delta)} + 4\pi e R \delta \,.
	\eeqan

\end{mytheorem}


\begin{myremark} 
As our confidence intervals do not require a bounded noise, likewise we can control the regret  with high probability without requiring  bounded observations, contrary to earlier works
such as \cite{Valko2013}.
\end{myremark}

\section{Discussion and related works}\label{sec:discussion}

\paragraph{Concentration results}

Theorem~\ref{thm:kLSmean} extends the self-normalized bounds of \cite{Abbasi2011} from the setting of linear function spaces to that of an RKHS with sub-Gaussian noise. Based on a nontrivial adaptation of the Laplace method, it yields self-normalized inequalities in a setting of possibly infinite dimension. It generalizes the following result of \cite{Wang2014} to kernel regression with $\lambda \neq \sigma^2$, which was already a generalization of a previous result by \cite{Srinivas2010} for bounded noise. It is also more general than the concentration result from \cite{Valko2013}, for kernel regression with $\lambda \neq \sigma^2$, which holds \emph{under the assumption of bounded observations}.

\begin{mylemma}[Proposition 1 from \cite{Wang2014}]
	\label{lem:gp_concentration:f_rkhs}
	Let $f_\star$ denote a function in the RKHS $\cK$ induced by kernel $k$ and let us define the posterior mean and variances with $\lambda = \sigma^2$, for (arbitrary) data $(x_{t'})_{t'\leq t}$. Assuming $\sigma$-sub-Gaussian noise variables, then for all $\delta' \in (0, 1)$ we have that
	\begin{align*}
	\Pr \big[ \exists x \in \cX : |f_{\lambda,t}(x) - f_\star(x)| \geq \ell_{\lambda,t+1}(\delta') k_{\lambda,t}^{1/2}(x,x) \big] \leq \delta',
	\quad \text{where}\quad \\
    \ell_{\lambda,t}^2(\delta')
	= \Vert f \Vert_\cK^2
	+ \sqrt{8 \gamma_{t-1}(\lambda) \ln \frac{2}{\delta'}}
	+ \sqrt{2 \ln \frac{4}{\delta'}} \Vert f \Vert_\cK   + 2 \gamma_{t-1}(\lambda)
	+ 2 \sigma \ln \frac{2}{\delta'}
	\end{align*}
	 and $\gamma_t(\lambda)=\frac{1}{2} \sum_{t'=1}^t \ln \big( 1 + \frac{1}{\lambda} k_{\lambda,t'-1}(x_{t'}, x_{t'}) \big)$ is the information gain.
\end{mylemma}
\begin{myremark}This results provides a bound 
	that is valid for each $t$, with probability higher $1-\delta$. In contrast, results from
	\cite{Abbasi2011}, as well as Theorem~\ref{thm:kLSmean} hold  with probability higher $1-\delta$,
	uniformly for all $t$, and are thus much stronger in this sense.
\end{myremark}

Theorem~\ref{thm:kLSmeantuned} extends Theorem~\ref{thm:kLSmean} to the case when \textit{the regularization is tuned online} based on gathered observations. To the best of our knowledge, no such result exists in the literature at the time of writing this paper. Moreover, Theorem~\ref{thm:rkhs_var_bound} provides variance estimates with confidence bounds scaling with $1/\sqrt{t}$, in the spirit of the results from \cite{Maurer2009}, that were provided in the  i.i.d. case. Thus, Theorem~\ref{thm:rkhs_var_bound} also appears to be new.
 Finally,  Corollary~\ref{cor:empBernstein} further specifies Theorem~\ref{thm:kLSmeantuned}  to the situation where the regularization is tuned according to Theorem~\ref{thm:rkhs_var_bound}, yielding a fully adaptive regularization procedure with explicit confidence bounds.

\paragraph{Bandits optimization}

When applied to the setting of multi-armed bandits, Theorems~\ref{thm:kernel_ts} and~\ref{thm:KernelUCB} respectively extend linear TS~\citep{Agrawal2014,Abeille2016} and UCB~\citep{Li2010,Chu2011} to the RKHS setting. Similar extensions have been provided in the literature: GP-UCB~\citep{Srinivas2010} generalizes UCB from the linear to the RKHS setting through the use of Gaussian processes; this corresponds to the case when $\lambda = \sigma^2$. The bounds they provide in the case when the target function belongs to an RKHS is however quite loose. KernelUCB~\citep{Valko2013} also generalizes UCB from the linear to the RKHS setting through the use of kernel regression. However the analysis of this algorithm was out of reach of their proof technique (that requires independence between arms) and they analyze instead the arguably less appealing variant called SupKernelUCB. Also, the analysis of both GP-UCB and  SupKernelUCB in the agnostic setting are respectively limited to bounded noise and bounded observations.

\section{Illustrative numerical experiments}
\label{sec:expes}

In this section, we illustrate the results introduced in the previous Sections~\ref{sec:kernelreg}
and \ref{sec:variance_estimation} on a few examples. The first one is the concentration result on the mean from Theorem~\ref{thm:kLSmean}, the second one is the variance estimate from Theorem~\ref{thm:rkhs_var_bound}, and the last one combines the formers by using the noise estimate to tune $\lambda_{t+1} = \sigma_t^2 / C^2$ in Theorem~\ref{thm:kLSmeantuned}, which corresponds to Corollary~\ref{cor:empBernstein}. We finally show the performance of kernelized bandits techniques using the provided variance estimates and adaptative regularization schemes.

We conduct the experiments using the function $f_\star$ shown by Figure~\ref{fig:function}, which has norm $\|f_\star\|_\cK = \| \theta_\star \|_2 = 2.06$ in the RKHS induced by a Gaussian kernel
$\quad   k(x, x') = e^{-\frac{(x - x')^2}{2 \rho^2}}\quad$
with length scale $\rho = 0.3$. We consider the space $\cX = [0, 1]$ and that the standard deviation of the noise is $\sigma = 0.1$. All further experiments use the upper-bound $C = 5$ on $\|f_\star\|_\cK$ and the lower-bound $\sigma_- = 0.01$ on $\sigma$.

\begin{figure}[H]
	\centering
	\includegraphics{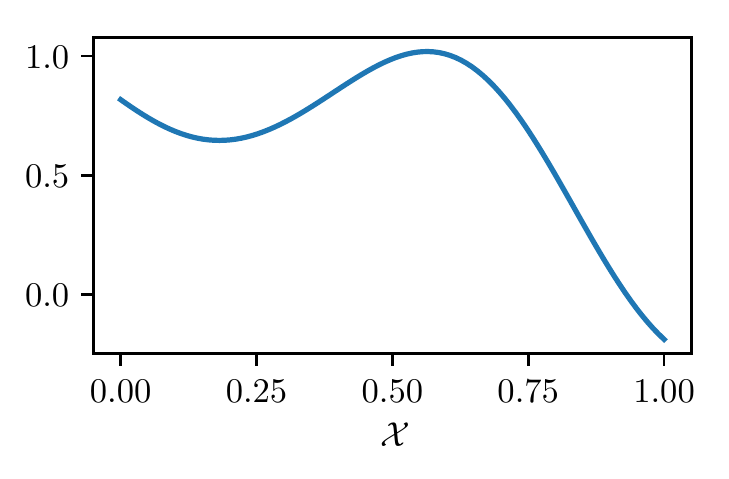}
	\vspace{-1mm}\caption{Test function $f_\star$ used in the following numerical experiments.}
	\label{fig:function}
\end{figure}

\subsection{Kernel concentration bound}

The following experiments compare the concentration result given by Theorem~\ref{thm:kLSmean} with the kernel concentration bounds from \cite{Wang2014} reported by Lemma~\ref{lem:gp_concentration:f_rkhs}. 
The true noise $\sigma = 0.1$ is assumed to be known and all observations are uniformly sampled from $\cX$. 
In both cases, we use a fixed confidence level $\delta=0.1$. 
Figure~\ref{fig:concentration_bounds_known_variance_againstwang} shows that for $\lambda = \sigma^2$, the result given by Theorem~\ref{thm:kLSmean} recovers the confidence envelope of \cite{Wang2014}. 
Note however that the confidence bound that we plot for Theorem~\ref{thm:kLSmean} are valid \textit{uniformly} over all time steps, while the one derived from \cite{Wang2014} is only valid separately for each time. 
Further, Theorem~\ref{thm:kLSmean} generalizes the latter result to the case where $\lambda \neq \sigma^2$. For illustration, Figure~\ref{fig:concentration_bounds_known_variance_thm1} illustrates the confidence envelopes in the special case where $\lambda = \sigma^2/C^2$, which also shows the potential benefit of such a tuning.

\begin{figure}
	\centering
	\includegraphics{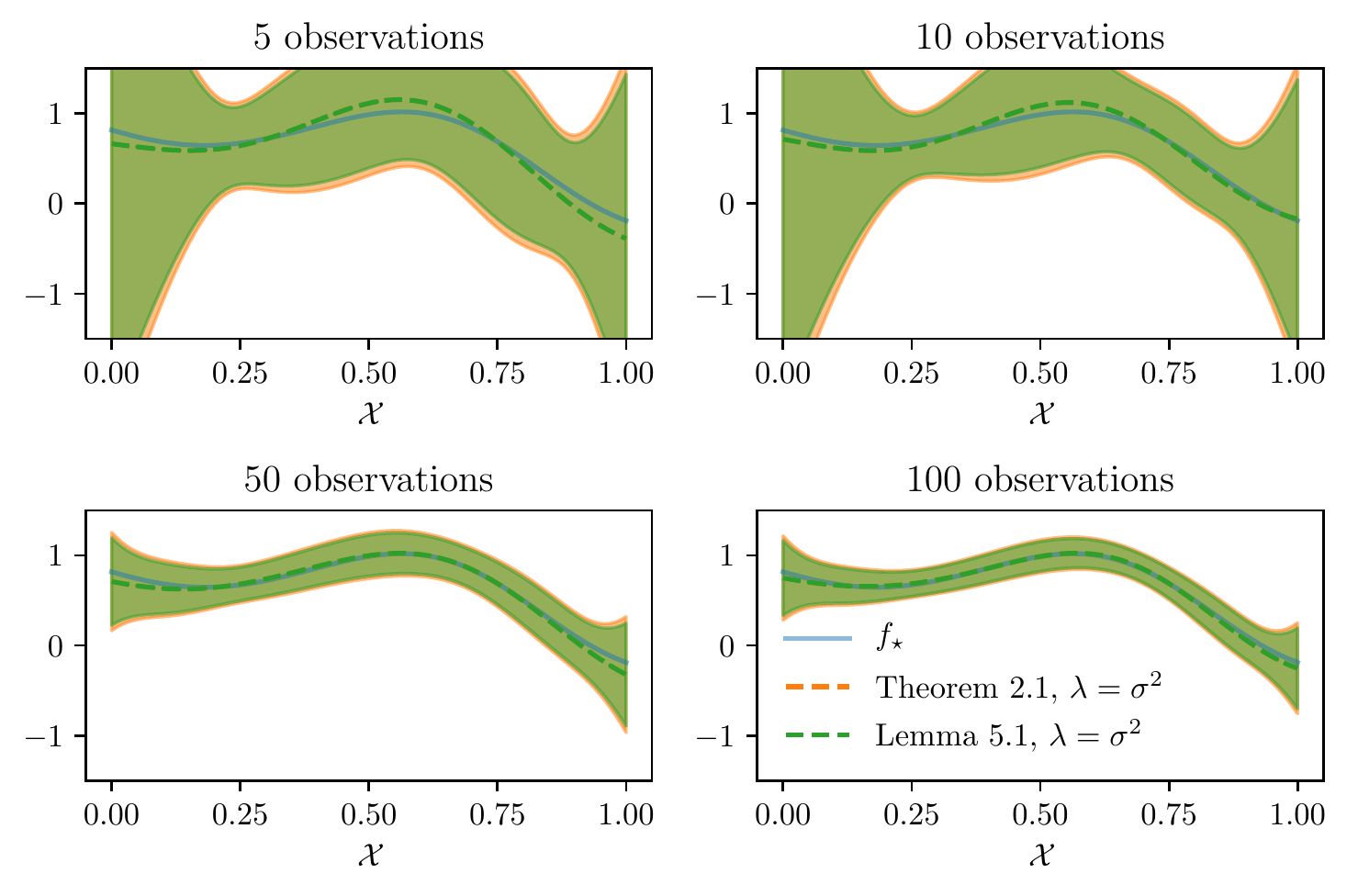}
	\vspace{-1mm}\caption{Confidence interval of Theorem~\ref{thm:kLSmean} and Lemma~\ref{lem:gp_concentration:f_rkhs}~\citep{Wang2014}.}
	\label{fig:concentration_bounds_known_variance_againstwang}
\end{figure}

\begin{figure}
	\centering
	\includegraphics{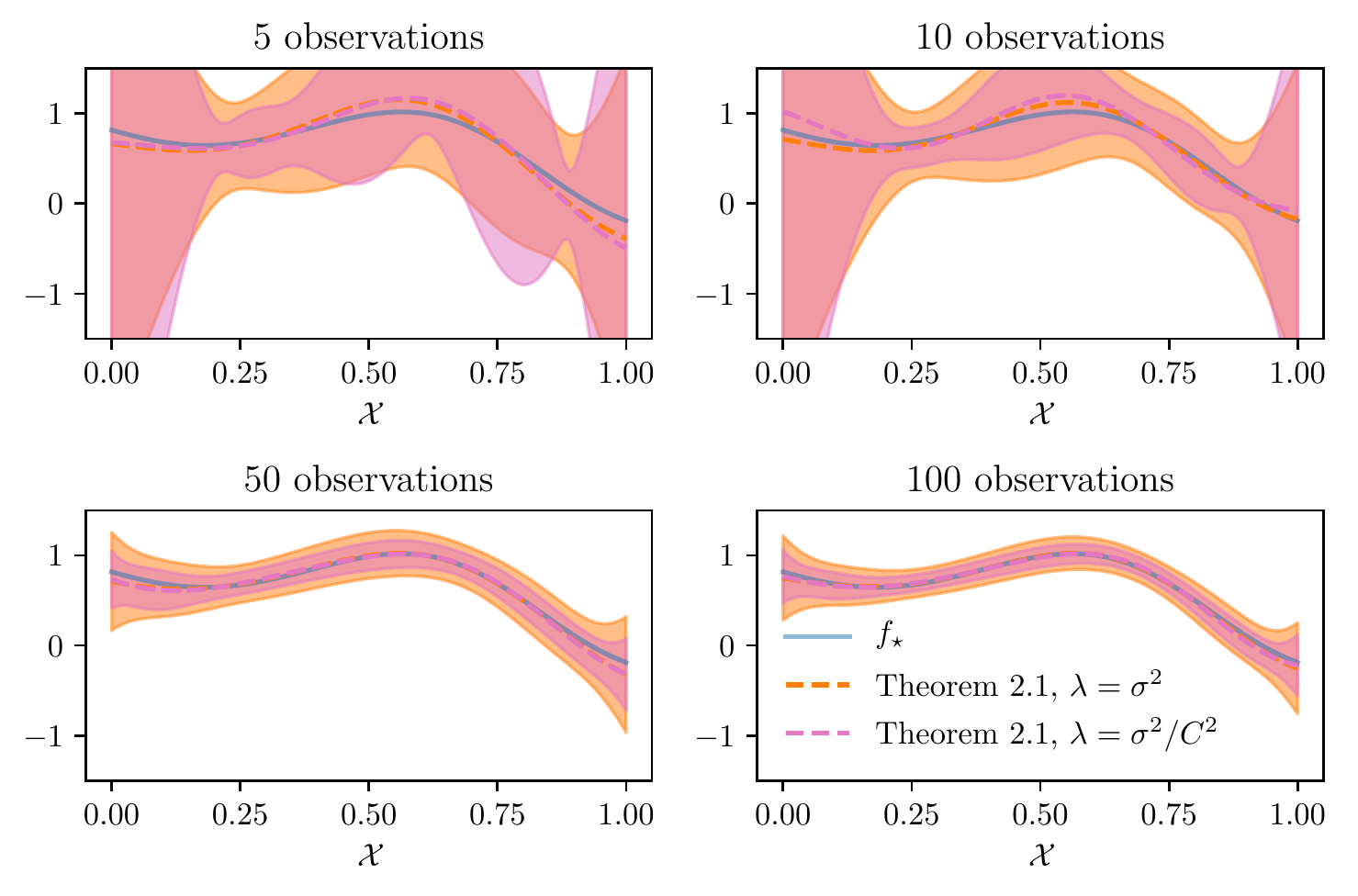}
	\vspace{-1mm}\caption{Confidence interval of Theorem~\ref{thm:kLSmean} for different $\lambda$.}
	\label{fig:concentration_bounds_known_variance_thm1}
\end{figure}

\subsection{Empirical variance estimate}
\label{sec:expes:empirical_variance_estimage}

We now illustrate the convergence rate of the noise estimates $\sigma_{-,t} = \max \{ \sigma_{-,t}(\lambda), \sigma_{-,t-1} \}$ and $\sigma_{+,t} = \min \{ \sigma_{+,t}(\lambda, \lambda_-), \sigma_{+,t-1} \}$ computed using Theorem~\ref{thm:rkhs_var_bound}, where $\lambda_-= \sigma_{-,t}^2/C^2$ and $\delta = 0.1$. All observations are uniformly sampled from $\cX$. Section~\ref{sec:variance_estimation} suggests that $\lambda = \sigma_{+,t-1}^2/C^2$ should provide tighter bounds than a fixed $\lambda = \sigma_+^2/C^2$. Figure~\ref{fig:empirical_noise_estimation_lambda} shows that this is indeed the case especially for large values of $t$. We also see that the adaptive update of $\lambda$ converges to the same value, whatever the initial bound $\sigma_+$. This is especially interesting when $\sigma_+$ is a loose initial upper bound on $\sigma$.

\begin{figure}
	\centering
	\begin{subfigure}{0.49\textwidth}
		\captionsetup{skip=-5pt}
		\includegraphics[width=\textwidth]{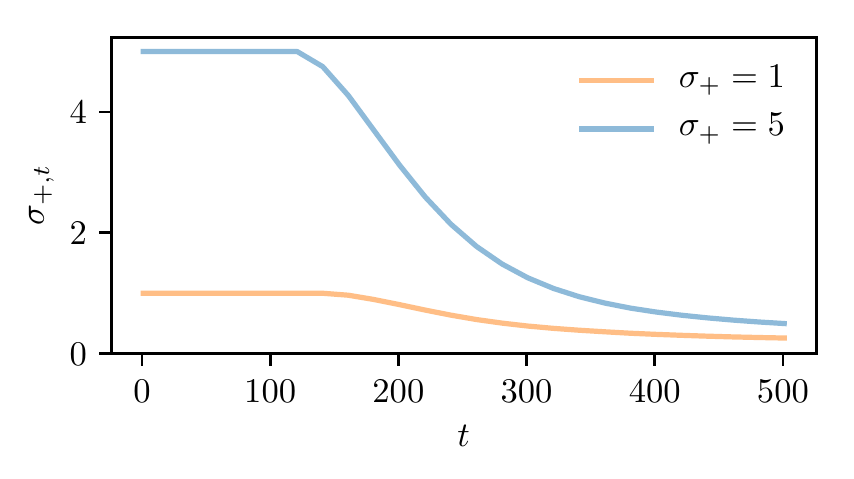}
		\caption{}
	\end{subfigure}
	\begin{subfigure}{0.49\textwidth}
		\captionsetup{skip=-5pt}
		\includegraphics[width=\textwidth]{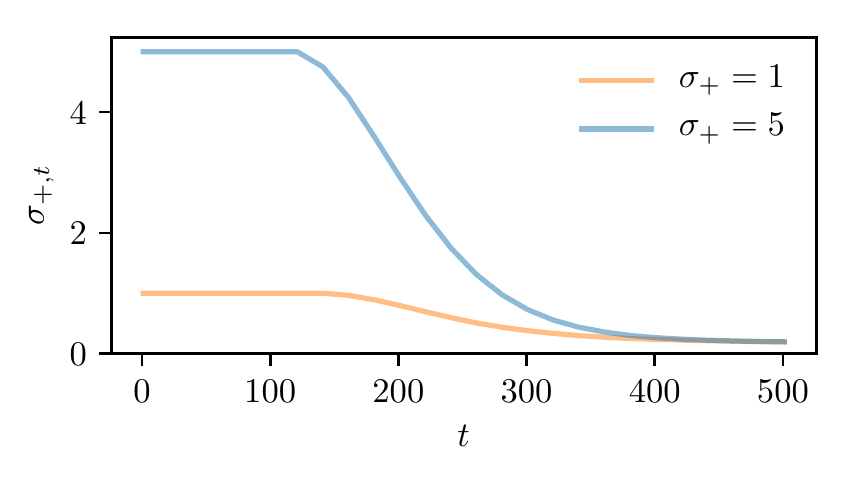}
		\caption{}
	\end{subfigure}
	\caption{Noise estimate from Theorem~\ref{thm:rkhs_var_bound} with $\sigma_+$ for a) fixed $\lambda = \sigma_+^2/C^2$; b) $\lambda = \sigma_{+, t-1}^2/C^2$.}
	\label{fig:empirical_noise_estimation_lambda}
\end{figure}

In practice, the bound of Theorem~\ref{thm:rkhs_var_bound} not using the knowledge of $\sigma_+$ may be useful even when $\sigma_+$ is known. This is illustrated by Figure~\ref{fig:empirical_noise_estimation_upperbound} that plots the upper-bound variance estimate $\sigma_{+,t}(\lambda, \lambda_-)$ for $\lambda = \sigma_{+, t-1}^2/C^2$ in both cases. In practice, we suggest to use the minimum of the bound using the knowledge of $\sigma_+$ (case~1) and of the agnostic one (case~2) to set $\sigma_{+,t}(\lambda, \lambda_-)$ and the maximum for $\sigma_{-,t}(\lambda)$. Figure~\ref{fig:empirical_noise_estimation} shows the resulting noise estimate envelopes for different $\sigma_+$ values (recall that $\sigma = 0.1$).

\begin{figure}
	\centering
	\begin{subfigure}{0.49\textwidth}
		\captionsetup{skip=-5pt}
		\includegraphics[width=\textwidth]{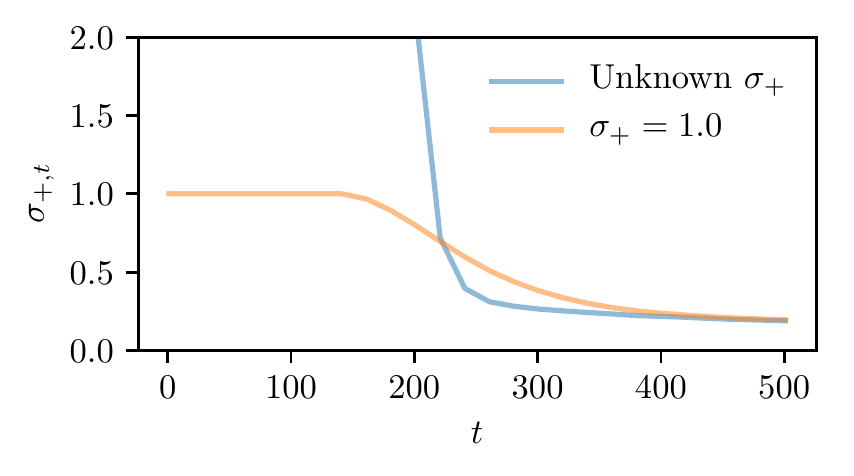}
		\caption{}
		\label{fig:empirical_noise_estimation_upperbound}
	\end{subfigure}
	\begin{subfigure}{0.49\textwidth}
		\captionsetup{skip=-5pt}
		\includegraphics[width=\textwidth]{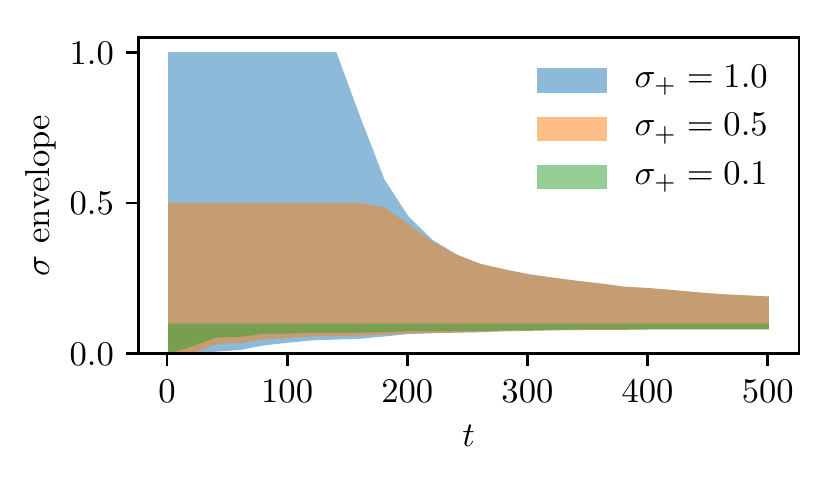}
		\caption{}
		\label{fig:empirical_noise_estimation}
	\end{subfigure}
	\caption{Variance estimate a) from Theorem~\ref{thm:rkhs_var_bound}, with and without $\sigma_+$; b) as minimum of the bounds and $\sigma_+$, for different upper-bounds.}
\end{figure}

\subsection{Adaptive regularization}

We now combine the previous experiments and use the estimated noise in order to tune the regularization. Recall that we consider $\sigma_{-,0} = \sigma_-$, $\sigma_{+,0} = \sigma_+$, and $\lambda_0 = \sigma_+^2/C^2$. On each time $t \geq 1$, we estimate the noise lower-bound $\sigma_{-,t} = \max \{ \sigma_{-,t}(\lambda_{t-1}), \sigma_{-,t-1} \}$ using Theorem~\ref{thm:rkhs_var_bound} and set $\lambda_- = \sigma_{-,t}^2/C^2$. 
We then compute the upper-bound noise estimate $\sigma_{+,t} = \min \{ \sigma_{+,t}(\lambda_{t-1}, \lambda_-), \sigma_{+,t-1} \}$ using Theorem~\ref{thm:rkhs_var_bound} and set $\lambda_t = \sigma_{+,t}^2/C^2$. We are now ready to compute the confidence interval given by Corollary~\ref{cor:empBernstein}. Note that $\delta = 0.1$ is used everywhere and all observations are uniformely sampled from $\cX$. Figure~\ref{fig:concentration_bounds_thm3} illustrates the resulting confidence envelope of this fully empirical model for noise upper-bound $\sigma_+ = 1$ (recall that the noise satisfies $\sigma = 0.1$) plotted against the confidence envelope obtained with Theorem~\ref{thm:kLSmean} with fixed $\lambda = \sigma_+^2/C^2$. We observe the improvement of the confidence intervals with the number of observations. Recall that this setting is especially challenging since the variance is unknown, the regularization parameter is tuned online, and the confidence bounds are valid uniformly over all time steps.

\begin{figure}[t]
	\centering
	\includegraphics{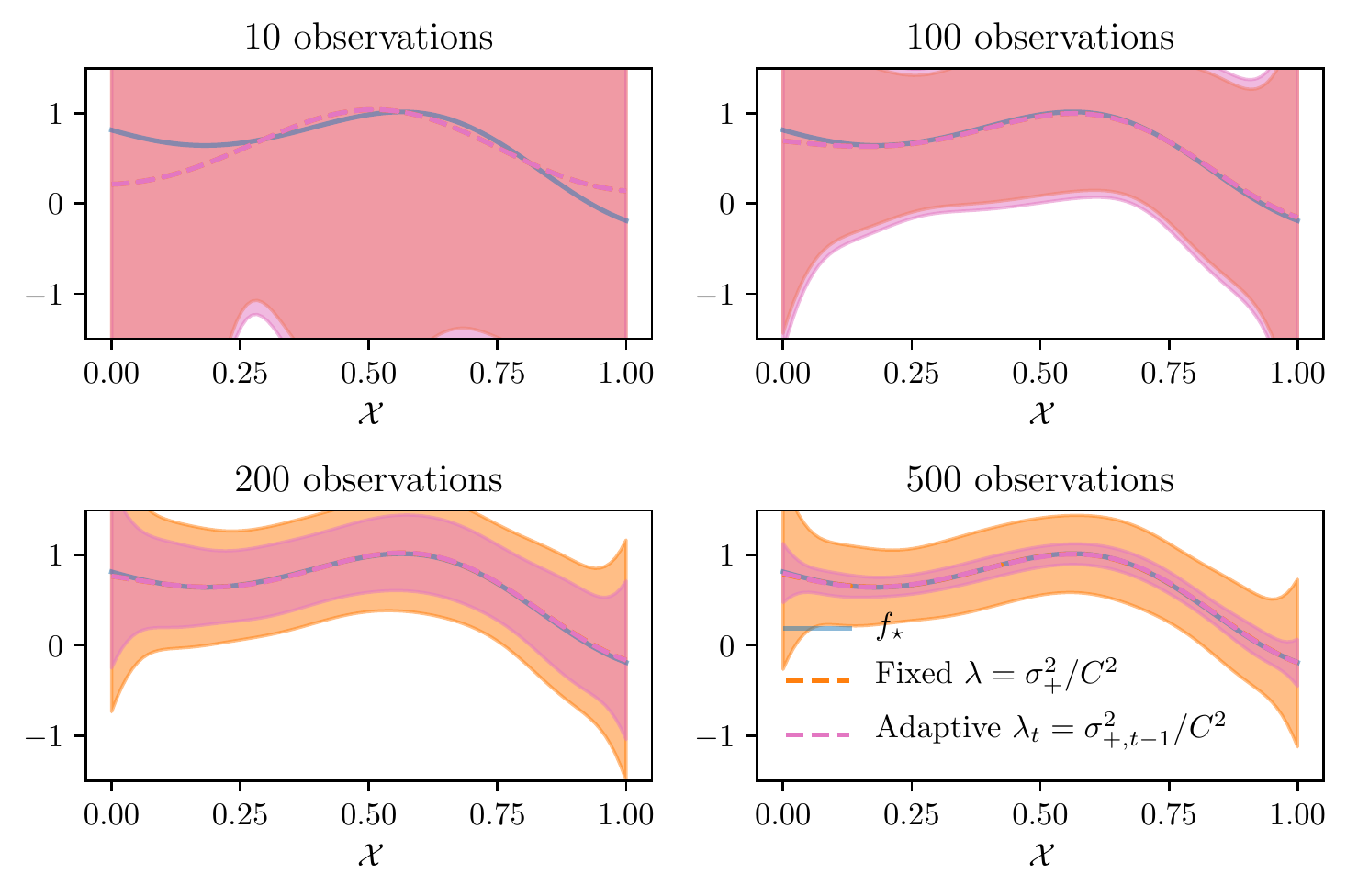}
	\caption{Confidence interval using fixed (Theorem~\ref{thm:kLSmean}) and adaptive (Corollary~\ref{cor:empBernstein}) regularization, for $\sigma_+ = 1$ and $\delta=0.1$.}
	\label{fig:concentration_bounds_thm3}
\end{figure}

\subsection{Kernelized bandits optimization}

In this section, we now evaluate the potential of kernelized bandits algorithms with variance estimate. We consider $\bX$ as the linearly discretized space $\cX = [0, 1]$ into 100 arms. Recall that the goal is to minimize the cumulative regret (Equation~\ref{eqn:regret}) and that we are optimizing the function shown by Figure~\ref{fig:function} with $\sigma = 0.1$. We evaluate Kernel UCB (Equation~\ref{eqn:kernel_ucb}) and Kernel TS (Algorithm~\ref{alg:kernel_ts} with $v_t = B_{\lambda_t, t-1}(\delta)/\sigma_{+,t-1}$) with three different configurations:
\begin{enumerate}[a),nolistsep]
	\item the oracle, that is with fixed $\lambda_t = \sigma^2/C^2$, assuming knowledge of $\sigma$;
	\item the fixed $\lambda_t = \sigma_+^2/C^2$, that is the best one can do without prior knowledge of $\sigma^2$;
	\item the adaptative regularization tuned with Corollary~\ref{cor:empBernstein}.
\end{enumerate}
All configurations use $C = 5$. Kernel UCB uses $\delta = 0.1/4$ and Kernel TS uses $\delta = 0.1/12$ such that their regret bounds respectively hold with probability $0.9$. Recall that observations are now sampled from $\bX$ using the bandits algorithms (they are not i.i.d.). Configurations b) and c) use $\sigma_+ = 1$, while the oracle a) uses $\sigma_+ = \sigma$. Figure~\ref{fig:bandits} shows the cumulative regret averaged over 100 repetitions. Note that the oracle corresponds to the best performance that could be expected by Kernel UCB and Kernel TS given knowledge of the noise. The plots confirm that adaptively tuning the regularization using the variance estimates can lead to a major improvement compared to using a fixed, non-accurate guess: after an initial burn-in phase, the regret of the adaptively tuned algorithm increases at the same rate as that of the oracle algorithm knowing the noise exactly. The fact that Kernel UCB outperforms Kernel TS much implies that inflating the variance in Kernel TS, as suggested per the theory presented previously, may not be optimal in practice. Further attention should be given to this question.

\begin{figure}[t]
	\centering
	\begin{subfigure}{0.49\textwidth}
		\captionsetup{skip=-5pt}
		\includegraphics[width=\textwidth]{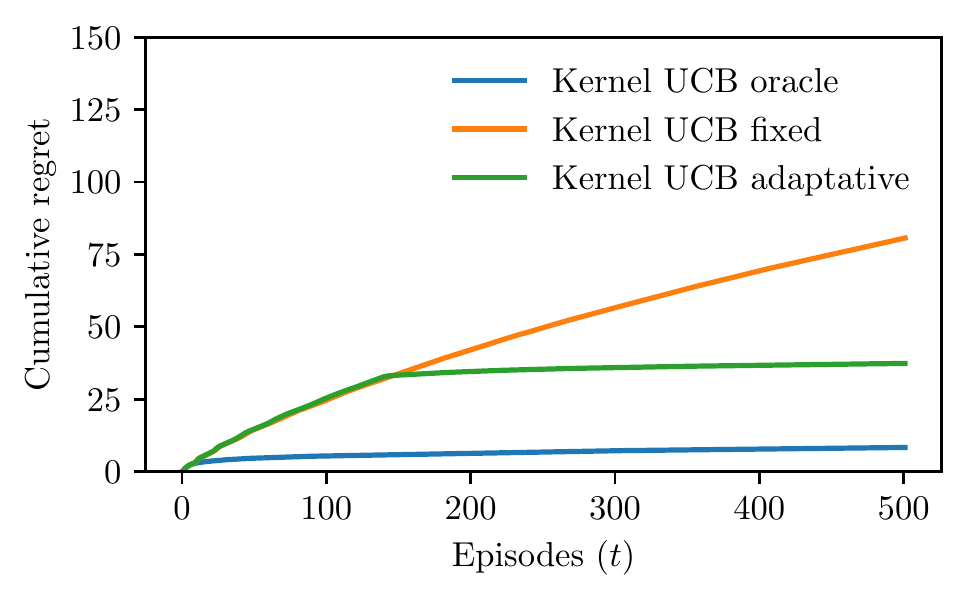}
		\caption{}
		\label{fig:bandits:kernel_ucb}
	\end{subfigure}
	\begin{subfigure}{0.49\textwidth}
		\captionsetup{skip=-5pt}
		\includegraphics[width=\textwidth]{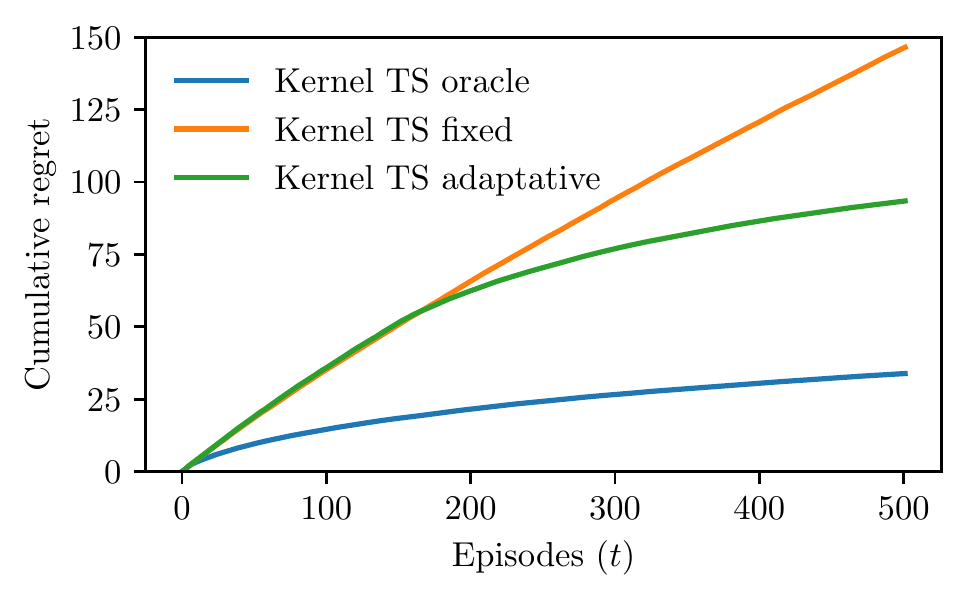}
		\caption{}
		\label{fig:bandits:kernel_ts}
	\end{subfigure}
	\caption{Averaged cumulative regret along episodes for a) Kernel UCB and b) Kernel TS.}
	\label{fig:bandits}
\end{figure}


\begin{figure}[t]
	\centering
	\includegraphics[width=0.5\textwidth]{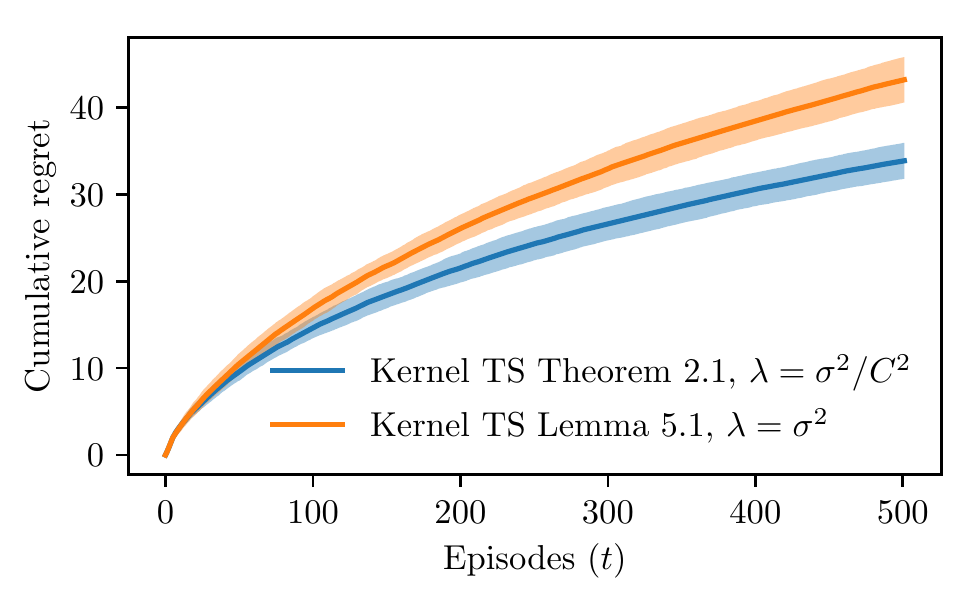}
	\caption{Averaged cumulative regret and one standard deviation along episodes for Kernel TS oracle with Theorem~\ref{thm:kLSmean} and Lemma~\ref{lem:gp_concentration:f_rkhs}~\citep{Wang2014}.}
	\label{fig:bandits:kernel_ts:against_wang}
\end{figure}

In order to evaluate the benefit of the concentration bound provided by Theorem~\ref{thm:kLSmean}, we compare the Kernel TS (Algorithm~\ref{alg:kernel_ts}) oracle using $v_t = B_{\lambda, t-1}/\sigma$ and $\lambda = \sigma^2/C^2$, where $B_{\lambda,t-1}$ is given by Theorem~\ref{thm:kLSmean}, against $v_t = \ell_t(\delta)$ where $\ell_t(\delta)$ is given by Lemma~\ref{lem:gp_concentration:f_rkhs}~\citep{Wang2014} with $\delta = 0.1$. Figure~\ref{fig:bandits:kernel_ts:against_wang} shows that the concentration bound given by Theorem~\ref{thm:kLSmean} improves the performance of Kernel TS compared with existing concentration results~\citep{Wang2014}. It highlights the relevance of expliciting the regularization parameter, which allows us to take advantage of regularization rates that may be better adapted.


\section{Conclusion}

This work addresses two problems: the online tuning of the regularization parameter in streaming kernel regression and the online estimation of the noise variance. To this extent, we introduce novel concentration bounds on the posterior mean estimate in streaming kernel regression with fixed and explicit regularization (Theorem~\ref{thm:kLSmean}), which we then extend to the setting where the regularization parameter is tuned (Theorem~\ref{thm:kLSmeantuned}). We further introduce upper- and lower-bound estimates of the noise variance (Theorem~\ref{thm:rkhs_var_bound}). Putting these tools together, we show how the estimate of the noise variance can be used to tune the kernel regularization in an online fashion (Corollary~\ref{cor:empBernstein}) while retaining theoretical guarantees. We also show how to use the proposed results in order to  derive kernelized variations of the most common bandits algorithms UCB and Thompson sampling, for which regret bounds are also provided (Theorems~\ref{thm:KernelUCB} and~\ref{thm:kernel_ts}).

All the proposed results and tools are illustrated through numerical experiments. The obtained results show the relevance of the introduced kernel regression concentration intervals for explicit regularization, which hold when the regularization does not correspond to the noise variance. The potential of the proposed regularization tuning procedure is illustrated through the application to kernelized bandits, where the benefits of adaptive regularization is undeniable when the noise variance is unknown (this is usually the case in practice). Finally, one must note that a major strength of the tools proposed in this work is to allow for an adaptively tuned regularization parameter while preserving theoretical guarantees, which is not the case when regularization is tuned
for example by cross-validation. 

Future work includes a natural extension of these techniques to obtain an empirical estimate of the kernel length scales. This information is often assumed to be known, while in practice it is often not available. Although some preliminary work has been done in that direction~\citep{Wang2014}, designing theoretically motivated algorithms addressing these concerns would help to fill an important gap between theory and practice. On a different matter, the current work gives the basis for performing Thompson sampling in RKHS, and could be extended to the contextual setting in a near future, as was done with CGP-UCB~\citep{Krause2011,Valko2013}.


\acks{This work was supported through funding from the Natural Sciences and Engineering Research Council of Canada (NSERC, Canada), the REPARTI strategic network (FRQ-NT, Qu\'ebec), MITACS, and E Machine Learning Inc. O.-A. M. acknowledges the support of the French Agence Nationale de la Recherche (ANR), under grant ANR-16-
	CE40-0002 (project BADASS).}


\newpage
\appendix

\section{Laplace method for tuned kernel regression}
\label{app:concentration}

In this section, we want to control the term $|f_{\lambda,t}(x)-f_\star(x)|$
simultaneously over all $t\leq T$. To this end, we resort to a version of the Laplace method carefully extended to the RKHS setting.

Before proceeding, we note that since $k:\cX\times\cX\to\Real$ is a kernel function (that is continuous, symmetric  positive definite) on a compact set $\cX$ equipped with a positive finite Borel measure $\mu$,
then there is an at most countable sequence $(\sigma_i,\psi_i)_{i\in\Nat^\star}$ where
$\sigma_i\geq 0$, $\lim_{i\to\infty} \sigma_i=0$ and $\{\psi_i\}$ form an orthonormal basis of $L_{2,\mu}(\cX)$, such that
\beqan
k(x,y) = \sum_{j=1}^\infty  \sigma_j \psi_j(x) \psi_j(y')\quad \text{ and }\quad
\norm{f}^2_\cK = \sum_{j=1}^\infty \frac{\langle f, \psi_j\rangle^2_{L_{2,\mu}}}{\sigma_j}
\eeqan
Let $\phi_i = \sqrt{\sigma_i}\psi_i$. Note that $\norm{\phi_i}_{L_2} = \sqrt{\sigma_i}$, $\norm{\phi_i}_\cK = 1$. 
Further, if $f = \sum_{i} \theta_i\phi_i$, then $\norm{f}^2_\cK = \sum_{i} \theta_i^2$ and 
$\norm{f}^2_{L_2} = \sum_{i} \theta_i^2 \sigma_i$.
In particular $f$ belongs to the RKHS if and only if $\sum_{i} \theta_i^2<\infty$.
For $\phi(x) = (\phi_1(x),\dots)$ and $\theta = (\theta_1,\dots)$, we now denote
$\theta^\top \phi(x)$ for $\sum_{i\in\Nat}\theta_i\phi_i(x)$, 
by analogy with the finite dimensional case. Note that $k(x,y) =  \phi(x)^\top\phi(y)$.

In the sequel, the following Martingale control will be a key component of the analysis.
	\begin{mylemma}[Hilbert Martingale Control]\label{lem:VecMart}
	Assume that the noise sequence $\{\xi_t\}_{t=0}^\infty$ is conditionally $\sigma^2$-sub-Gaussian
	\beqan
	\forall t\in\Nat, \forall \gamma \in\Real,\quad \ln \Esp[\exp(\gamma \xi_t) | \cH_{t-1}] \leq \frac{\gamma^2\sigma^2}{2}\,.
	\eeqan
	Let $\tau$ be a stopping time with respect to the filtration $\{\cH_{t}\}_{t=0}^\infty$ generated by the variables $\{x_t,\xi_t\}_{t=0}^\infty$.
	For any ${\bf q} = (q_1,q_2,\dots)$ such that ${\bf q}^\top\phi_i(x) =\sum_{i\in\Nat}q_i\phi(x)<\infty$,
	and deterministic positive $\lambda$, let us denote
	\beqan
	M_{m,\lambda}^{\bf q} = 
	\exp\bigg( \sum_{t=1}^m \frac{{\bf q}^\top\phi(x_t)}{\sqrt{\lambda}}\xi_t - \frac{\sigma^2}{2} \sum_{t=1}^m \frac{({\bf q}^\top \phi(x_t))^2}{\lambda}\bigg)
	\eeqan
	Then, for all such ${\bf q}$ the quantity $M_{\tau,\lambda}^{\bf q}$ is well defined and satisfies 
	\beqan
	\ln \Esp[M_{\tau,\lambda}^{\bf q}]\leq 0\,.
	\eeqan
\end{mylemma}
\begin{proof}
	The only difficulty in the proof is to handle the stopping time.
	Indeed, for all $m\in\Nat$, thanks to the conditional $R$-sub-Gaussian property, it is immediate to show that $\{M_{m,\lambda}^{\bf q}\}_{m=0}^\infty$ is a non-negative super-martingale and actually satisfies $\ln \Esp[M_{m,\lambda}^{\bf q}]\leq 0$.
	
	By the convergence theorem for nonnegative super-martingales, $M_\infty^{\bf q} = \lim_{m\to\infty}M_{m,\lambda}^{\bf q}$ is almost surely well-defined,
	and thus $M_{\tau,\lambda}^{\bf q}$ is well-defined (whether $\tau<\infty$ or not) as well.
	In order to show that $\ln \Esp[M_{\tau,\lambda}^{\bf q}]\leq 0$, we introduce
	a stopped version $Q_m^{\bf q} = M_{\min\{\tau,m\},\lambda}^{\bf q}$ of $\{M_{m,\lambda}^{\bf q}\}_m$.
	Now $\Esp[M_{\tau,\lambda}^{\bf q}] = \Esp[\liminf_{m\to\infty}Q_m^{\bf q}]\leq \liminf_{m\to\infty}\Esp[Q_m^{\bf q}] \leq 1$ by Fatou's lemma, which concludes the proof. We refer to \citep{Abbasi2011} for further details.
\end{proof}

We are now ready to prove the following result.


\begin{proofof}{Theorem~\ref{thm:kLSmeantuned} (Streaming Kernel Least-Squares)}
	We make use of the features in an explicit way.  Let $\lambda=\lambda_{t+1}$.
	For $f_\star\in\cK$, we denote $\theta^\star$ its corresponding parameter sequence.	
	We let ${\bf \Phi}_{t} = (\phi(x_{t'}))_{t'\leq t}$ be a $t\times \infty$ matrix built from the features and introduce the bi-infinite matrix $V_{\lambda,t} = I + \frac{1}{\lambda}{\bf \Phi}_{t}^\top{\bf \Phi}_{t}$ as well as the noise vector $E_t = (\xi_1,\dots,\xi_t)$.
	In order to control the term $| f_{\blambda,t}-f_\star(x)|$,  we first decompose the estimation term. Indeed, using the feature map, it holds that
\beqan
f_{\blambda,t}(x) &=& k_t(x)^\top({\bf K_t}+ \lambda I_t)^{-1}Y_t\\
&=& \phi(x)^\top{\bf \Phi}_{t}^\top  ( {\bf \Phi}_{t}{\bf \Phi}_{t}^\top  + \lambda I_t)^{-1}Y_t\\
&=& \phi(x)^\top{\bf \Phi}_{t}^\top  \bigg( 
\frac{I_t}{\lambda} -
\frac{1}{\lambda}{\bf \Phi}_{t}\big(\lambda I\! +\! {\bf \Phi}_{t}^\top{\bf \Phi}_{t}\big)^{-1}
{\bf \Phi}_{t}^\top\bigg)Y_t\\
&=&\phi(x)^\top ( {\bf \Phi}_{t}^\top {\bf \Phi}_{t} + \lambda I)^{-1}{\bf \Phi}_{t}^\top 
( {\bf \Phi}_{t}\theta^\star + E_{t})
\eeqan
where in the third line, we used the Shermann-Morrison formula. From this, simple algebra yields
\beqan
f_{\blambda,t}(x)-f^\star(x)&=& \frac{1}{\lambda}\phi(x)^\top V_{\lambda,t}^{-1}\big({\bf \Phi}_{t}^\top E_{t}-\lambda\theta^\star\big)\,.
\eeqan
	We then obtain, from a simple H\"{o}lder inequality using the appropriate matrix norm,
	the following decomposition, that is valid provided that all terms involved are finite.
	\beqan
	|f_{\blambda,t}(x) - f(x)| \leq \frac{1}{\sqrt{\lambda}}\|\phi(x)\|_{V_{\lambda,t}^{-1}}\bigg[ \frac{1}{\sqrt{\lambda}}\|{\bf \Phi}_{t}^\top E_t\|_{V_{\lambda,t}^{-1}} + \sqrt{\lambda}\|\theta^\star\|_{ V_{\lambda,t}^{-1}} \bigg]
	\eeqan

	Now, we note that a simple application of the Shermann-Morrison formula yields
	\beqan
	\|\phi(x)\|^2_{V_{\lambda,t}^{-1}} =k_{t}(x,x)\,.
	\eeqan
	On the other hand, the last term of the bound is controlled as
	$\|\theta^\star\|_{ V_{\lambda,t}^{-1}}  \leq \|\theta^\star\|$. 
Thus, 
	\beqan
|f_{\blambda,t}(x) - f(x)| \leq \frac{k_{\blambda,t}^{1/2}(x,x)}{\sqrt{\lambda_{t+1}}}\bigg[\frac{1}{\sqrt{\lambda_{t+1}}}\|{\bf \Phi}_{t}^\top E_t\|_{V_{\lambda_{t+1},t}^{-1}}  + \sqrt{\lambda_{t+1}}\|\theta^\star\|_2 \bigg]\,.
	\eeqan

	In order to control the remaining term,
	$\frac{1}{\sqrt{\lambda_{t+1}}}\|{\bf \Phi}_{t}^\top E_t\|_{V_{\lambda_{t+1},t}^{-1}}$, for all $t$, we now want to apply Lemma~\ref{lem:VecMart}. However, the lemma does not apply since $\lambda_{t+1}$ is $\cH_t$-measurable.
Thus, before proceeding, we upper-bound it by the similar expression involving $\lambda_\star$:
	\beqan
	\frac{1}{\lambda}\|{\bf \Phi}_{t}^\top E_t\|_{V_{\lambda,t}^{-1}}^2
	&=& E_t^\top\frac{{\bf \Phi}_{t}^\top}{\lambda} (I+ \frac{1}{\lambda}{\bf \Phi}_{t}^\top{\bf \Phi}_{t})^{-1} \frac{{\bf \Phi}_{t}}E_t\\
	&=& E_t^\top {\bf \Phi}_{t}^\top(\lambda I+ {\bf \Phi}_{t}^\top{\bf \Phi}_{t})^{-1} {\bf \Phi}_{t}E_t\\
	&\leq& E_t^\top{\bf \Phi}_{t}^\top(\lambda_\star I+ {\bf \Phi}_{t}^\top{\bf \Phi}_{t})^{-1} {\bf \Phi}_{t}E_t\,,	
	\eeqan
	where  in the last line, we use the fact that the function $f:\lambda\to u^\top(\lambda I+A)^{-1}u$, for 
	$u=   {\bf \Phi}_{t}E_t$ and $A={\bf \Phi}_{t}^\top{\bf \Phi}_{t}$ is non increasing 
	(see Lemma~\ref{lem:increasingnorm} below). 	
	Thus, $\frac{1}{\sqrt{\lambda_{t+1}}}\|{\bf \Phi}_{t}^\top E_t\|_{V_{\lambda_{t+1},t}^{-1}} \leq \frac{1}{\sqrt{\lambda_\star}}\|{\bf \Phi}_{\lambda_\star,t}^\top E_t\|_{V_{\lambda_\star,t}^{-1}}$. 
	Next, we introduce a random stopping time $\tau$, to be defined later and apply Lemma~\ref{lem:VecMart}.

	More precisely, let $Q\sim\cN(0,I)$ be an infinite Gaussian random sequence which is independent of all other random variables.
	We denote $Q^\top\phi(x) = \sum_{i\in\Nat} Q_i\phi_i(x)$.
	For all $x$, $k(x,x)=\sum_{i\in\Nat} \phi_i^2(x)<\infty$ and thus  $\Var( Q^\top {\bf \phi}(x) )<\infty$.
	We define $M_{m,\lambda_\star} = \Esp[M_{m,\lambda_\star}^Q]$.
	Clearly, we still have $\Esp[M_{\lambda_\star,\tau}] = \Esp[\Esp[M_{m,\lambda_\star}^Q]|Q] \leq 1$. 
	Since  $V_{\lambda_\star,\tau} = I + \frac{1}{\lambda_\star} {\bf \Phi}_{\tau}^\top{\bf \Phi}_{\tau}$, elementary algebra gives
	\beqan
	\det(V_{\lambda_\star,\tau}) &=& \det(V_{\lambda_\star,\tau-1} + \frac{1}{\lambda_\star}\phi(x_\tau)\phi(x_\tau)^\top)=\det(V_{\lambda_\star,\tau-1})(1+ \frac{1}{\lambda_\star}\norm{\phi(x_\tau)}_{V_{\lambda_\star,\tau-1}^{-1}}^{2})\\
	&=&\det(V_{\lambda_\star,0})\prod_{t'=1}^\tau\bigg(1+ \frac{1}{\lambda_\star}\norm{\phi(x_{t'})}_{V_{\lambda_\star,t'-1}^{-1}}^{2}\bigg)\,,
	\eeqan
	where we used the fact that the eigenvalues of a matrix of the form $I+xx^\top$ are all ones except for the eigenvalue $1+\norm{x}^2$ corresponding to $x$.
	Then, note that $\det(V_{\lambda_\star,0})=1$ and thus 
	\beqan
	\ln(\det(V_{\lambda_\star,\tau})) &=&  \sum_{t'=1}^\tau\ln(1 + \frac{1}{\lambda_\star}\norm{\phi(x_{t'})}_{V_{\lambda_\star,t'-1}^{-1}}^{2})\\
	&=&\frac{1}{2}\sum_{t'=1}^\tau \ln\bigg(1 +  \frac{1}{\lambda_\star}k_{\lambda_\star,t'-1}(x_{t'},x_{t'})\bigg)\,.
	\eeqan
	In particular, $\ln(\det(V_{\lambda_\star,\tau}))$ is finite.
	The only difficulty in the proof is now to handle the possibly infinite dimension. 
	To this end, it is enough to take a look at the approximations using the $d$ first dimension of the sequence for each $d$.  We note $Q_d,M_{\lambda_\star,\tau},\Phi_{\tau,d}$ and $V_{\tau,d}$
	the restriction of the corresponding quantities to the components $\{1,\dots,d\}$.
	Thus $Q_d$ is  Gaussian $\cN(0,I_d)$. Following the steps from \cite{Abbasi2011},
	we obtain that 
	\beqan
	M_{m,d,\lambda_\star} = \frac{1}{\det(V_{\lambda_\star,m,d})^{1/2}}\exp\bigg(\frac{1}{2\sigma^2\lambda_\star} \norm{{\bf \Phi}_{m,d}^\top E_m}^2_{V_{\lambda_\star,m,d}^{-1}}\bigg)\,.
	\eeqan
	Note also that $\Esp[M_{\tau,d,\lambda_\star}] \leq 1$ for all $d\in\Nat$.
	Thus, we obtain by an application of Fatou's lemma that
	\beqan
	\Pr\bigg(\lim_{d\to\infty} \frac{\norm{{\bf \Phi}_{\tau,d}^\top E_\tau}^2_{V_{\lambda_\star,\tau,d}^{-1}}}{2\sigma^2\lambda_\star\log\Big(\det(V_{\lambda_\star,\tau,d})^{1/2}/\delta\Big)}>1 \bigg) &\leq &
	\Esp\bigg[\lim_{d\to\infty} \frac{\delta\exp\bigg(\frac{1}{2\lambda_\star\sigma^2} \norm{{\bf \Phi}_{\tau,d}^\top E_\tau}^2_{V_{\tau,d}^{-1}}\bigg)}{\det(V_{\lambda_\star,\tau,d})^{1/2}}\bigg]\\
	&\leq& \delta \lim_{d\to\infty}\Esp[M_{\tau,d,\lambda_\star}] \leq \delta\,.
	\eeqan
	
	We conclude by defining $\tau$ following \cite{Abbasi2011}, by 
	\beqan
	\tau(\omega) = \min\bigg\{ t \geq 0; \omega \in\Omega \text{ s.t. } 
\|{\bf \Phi}_{t}^\top E_t\|^2_{V_{\lambda_\star,t}^{-1}}> 2\sigma^2\lambda_\star\log\Big(\det(V_{\lambda_\star,t})^{1/2}/\delta\Big)\bigg\}\,.
	\eeqan
	Then $\tau$ is a random stopping time and 
	\beqan
	\Pr\bigg( \exists t, \|{\bf \Phi}_{t}^\top E_t\|^2_{V_{\lambda_\star,t}^{-1}}> 2\sigma^2\lambda_\star\log\Big(\det(V_{\lambda_\star,t})^{1/2}/\delta\Big)\bigg) = \Pr(\tau<\infty) \leq \delta.
	\eeqan	
	Finally, combining this result with the previous remarks we obtain that with probability higher than $1-\delta$,  uniformly over $x\in\cX$
	and $t\leq T$, it holds that
	\beqan|
	f_{\blambda,t}-f_\star(x)| \leq 
	\frac{k_{\blambda,t}^{1/2}(x,x)}{\sqrt{\lambda_{t+1}}}\bigg[
	\sqrt{2\sigma^2\ln\bigg(\frac{\det(I+\frac{1}{\lambda_\star}\Phi_{t}^\top\Phi_{t})^{1/2}}{\delta}\bigg)}
	+\sqrt{\lambda_{t+1}}\norm{f_\star}_{\cK}
	\bigg]\,.
	\eeqan
\end{proofof}

\begin{mylemma}[Technical lemma]\label{lem:increasingnorm}
	The function $f:\lambda\mapsto u^\top (\lambda I + A)^{-1} u$, where $A$ is a semi-definite positive matrix  and $u$ is any vector, is non-decreasing on
	$\lambda\in\Real^+$.
\end{mylemma}
\begin{proof}
	Indeed, let $h>0$. By the Sherman-Morrison formula, we obtain
	\beqan
	f(\lambda+h) = f(\lambda) - h u^\top(\lambda I + A)^{-1}(I+ h (\lambda I + A)^{-1})^{-1} (\lambda I + A)^{-1}u\,.
	\eeqan
	Thus, since $\lambda I+A$ is also semi-definite positive, we have
	\beqan
	 \lim_{h\to 0}\frac{f(\lambda+h)-f(\lambda)}{h} = - u^\top(\lambda I + A)^{-1}(\lambda I + A)^{-1}u \leq 0\,.
	 \eeqan
\end{proof}

\section{Variance estimation}
\label{app:variance_estimation}

In this section, we give the proof of Theorem~\ref{thm:rkhs_var_bound}. To this end, we proceed in two steps. First, we provide an upper bound and lower bound on the variance estimate in the next theorem. Then, we use these bounds in order to derive the final statement.

\begin{mytheorem}[Regularized variance estimate]
\label{thm:kLSvar}
	Under the second-order sub-Gaussian predictable assumption, for any random stopping time $\tau$ for the filtration of the past, with probability higher than $1-3\delta$, it holds
	\beqan
	\sqrt{\hat \sigma^2_{k,\lambda,\tau}} &\leq& 
	\sigma\bigg[1 + \sqrt{\frac{2C_\tau(\delta)}{\tau}} \bigg]    
	+ \|f^\star\|_\cK\sqrt{\frac{\lambda}{\tau}}  \sqrt{1 - 
		\frac{1}{\max_{t\leq \tau}(1+k_{\lambda,t-1}(x_t,x_t))}}
	\\   
	\sqrt{\hat \sigma^2_{k,\lambda,\tau}} &\geq& 
	\sigma\bigg[1-\sqrt{\frac{C_\tau(\delta)}{\tau}}-
	\sqrt{\frac{C_\tau(\delta) + 2D_{\lambda_\star,\tau}(\delta)}{\tau}}\bigg]
	-\sqrt{\frac{2 \sigma \lambda^{1/2}\|f^\star\|_\cK\sqrt{D_{\lambda_\star,\tau}(\delta)}}{\tau}}\,.
	\eeqan
	where we introduced for convenience the constants 
	$C_\tau(\delta)=\ln(e/\delta)\big[1+\ln(\pi^2\ln(\tau)/6)/\ln(1/\delta)\big]$
	and $D_{\lambda_\star,\tau}(\delta)= 2\ln(1/\delta) + \sum_{t=1}^\tau\ln(1\! +\! \frac{1}{\lambda_\star}k_{\lambda_\star,t-1}(x_t,x_t))$.   
\end{mytheorem}

\begin{proof}
	We use the feature maps and start with the following decomposition
	\beqa
	\lefteqn{
		\tau\hat \sigma^2_{k,\lambda,\tau} = \sum_{t=1}^\tau (y_t- f_{\lambda,\tau}(x_t))^2 =\sum_{t=1}^\tau(y_t- \langle \theta_{\lambda,\tau},\phi(x_t)\rangle)^2}\nonumber\\
	&=&
	(\theta^\star -\theta_{\lambda,\tau})^\top G_\tau (\theta^\star -\theta_{\lambda,\tau})
	+ \|E_\tau\|^2 + 2 (\theta^\star -\theta_{\lambda,\tau})^\top \Phi_{\tau}^\top E_\tau\,.
	\label{eqn:decRegVar}
	\eeqa
	where 
	$	\theta^\star -\theta_{\lambda,\tau} = (I- G_{\lambda,\tau}^{-1}G_\tau)\theta^\star - G_{\lambda,\tau}^{-1}{\bf \Phi}_{\tau}^\top E_\tau$
	with  $G_{\lambda,\tau}=\lambda I +  G_\tau$ and 
	$G_\tau= {\Phi}_{\tau}^\top {\Phi}_{\tau}$.
	
	On the one hand, we can control the first term in 
	\eqref{eqn:decRegVar} via
	\beqan
	\lefteqn{
		(\theta^\star -\theta_{\lambda,\tau})^\top G_\tau(\theta^\star -\theta_{\lambda,\tau})}\\
	&=&[(I-G_{\lambda,\tau}^{-1} G_\tau)\theta^\star - G_{\lambda,\tau}^{-1} \Phi_\tau^\top E_\tau]^\top G_\tau[(I-G_{\lambda,\tau}^{-1} G_\tau)\theta^\star - G_{\lambda,\tau}^{-1} \Phi_\tau^\top E_\tau]\\
	&=& [\lambda \theta^\star-\Phi_\tau^\top E_\tau]^\top G_{\lambda,\tau}^{-1}G_\tau G_{\lambda,\tau}^{-1}[\lambda \theta^\star-\Phi_\tau^\top E_\tau]\\
	&=& [\lambda \theta^\star-\Phi_\tau^\top E_\tau]^\top [G_{\lambda,\tau}^{-1}-
	\lambda G_{\lambda,\tau}^{-2}][\lambda \theta^\star-\Phi_\tau^\top E_\tau]\\
	&=&\|\Phi_\tau^\top E_\tau\|_{G_{\lambda,\tau}^{-1}}^2 - \lambda\|\Phi_\tau^\top E_\tau\|_{G_{\lambda,\tau}^{-2}}^2 + \lambda^2\|\theta^\star\|_{G_{\lambda,\tau}^{-1}}^2-
	\lambda^3\|\theta^\star\|_{G_{\lambda,\tau}^{-2}}^2\\
	&&-2\lambda {\theta^\star}^\top[G_{\lambda,\tau}^{-1}-
	\lambda G_{\lambda,\tau}^{-2}]\Phi_\tau^\top E_\tau
	\eeqan
	where we used the fact that $I-G_{\lambda,\tau}^{-1} G_\tau = \lambda G_{\lambda,\tau}^{-1}$
	and then that $G_{\lambda,\tau}^{-1} G_\tau G_{\lambda,\tau}^{-1}=G_{\lambda,\tau}^{-1}-
	\lambda G_{\lambda,\tau}^{-2}$.
	Likewise, we control the third term in 
	\eqref{eqn:decRegVar} via
	\beqan
	2 (\theta^\star -\theta_{\lambda,\tau})^\top \Phi_\tau^\top E_\tau &=&
	2[(I-G_{\lambda,\tau}^{-1} G_\tau)\theta^\star - G_{\lambda,\tau}^{-1} \Phi_\tau^\top E_\tau]^\top\Phi_\tau^\top E_\tau \\
	&=& 2[\lambda \theta^\star-\Phi_\tau^\top E_\tau]^\top G_{\lambda,\tau}^{-1}\Phi_\tau^\top E_\tau\\
	&=& 2 \lambda{\theta^\star}^\top G_{\lambda,\tau}^{-1}\Phi_\tau^\top E_\tau - 2\|\Phi_\tau^\top E_\tau\|_{G_{\lambda,\tau}^{-1}}^2\,.
	\eeqan

	Combining these two bounds, we have
	\beqan
	\lefteqn{
		\sum_{t=1}^\tau(y_t- \langle \theta_{\lambda,\tau},\phi(x_t)\rangle)^2}\\ 
	&=& \|E_\tau\|^2 -\|\Phi_\tau^\top E_\tau\|_{G_{\lambda,\tau}^{-1}}^2 - \lambda \|\Phi_\tau^\top E_\tau\|_{G_{\lambda,\tau}^{-2}}^2\\
	&&  + \lambda^2\|\theta^\star\|_{G_{\lambda,\tau}^{-1}}^2-
	\lambda^3\|\theta^\star\|_{G_{\lambda,\tau}^{-2}}^2
	+2\lambda^2 {\theta^\star}^\top G_{\lambda,\tau}^{-2}\Phi_\tau^\top E_\tau\\
	&\leq& \|E_\tau\|^2+ \frac{\lambda^2}{\lambda_{\min(G_{\lambda,\tau})}} \|\theta^\star\|^2_2\Big(1 - \frac{\lambda}{\lambda_{\max}(G_{\lambda,\tau})}\Big)
	+ 2\frac{\lambda^2}{\lambda_{\min^{3/2}(G_{\lambda,\tau})}}\|\theta^\star\|_2
	\|\Phi_\tau^\top E_\tau\|_{G_{\lambda,\tau}^{-1}}\\
	&\geq&
	\|E_\tau\|^2+ \frac{\lambda^2}{\lambda_{\max}(G_{\lambda,\tau})} \|\theta^\star\|^2_2\Big(1 - \frac{\lambda}{\lambda_{\min(G_{\lambda,\tau})}}\Big)
	- 2\frac{\lambda^2}{\lambda_{\min^{3/2}(G_{\lambda,\tau})}}\|\theta^\star\|_2
	\|\Phi_\tau^\top E_\tau\|_{G_{\lambda,\tau}^{-1}}\\
	&&-\|\Phi_\tau^\top E_\tau\|_{G_{\lambda,\tau}^{-1}}^2\Big(1 + \frac{\lambda}{\lambda_{\min(G_{\lambda,\tau})}}\Big)\,.
	\eeqan
	
	Now, from Lemma~\ref{lem:varconcentration}, it holds on an event $\Omega_1$ 
	of probability higher than $1-\delta$,
	\beqan
	0 \leq 
	\|\Phi_\tau^\top E_\tau\|_{G_{\lambda,\tau}^{-1}}^2
	=	\frac{1}{\lambda}\|\Phi_{\tau}^\top E_\tau\|_{V_{\lambda,\tau}^{-1}}^2 \leq 
	\frac{1}{\lambda_\star}\|\Phi_{\tau}^\top E_\tau\|_{V_{\lambda_\star,\tau}^{-1}}^2\leq \sigma^2 D_{\lambda_\star,\tau}(\delta)\,.
	\eeqan
	
	On the other hand, we control the second term $\|E_\tau\|^2$ by Lemma~\ref{lem:varconcentration} below, and obtain that with probability higher than $1-2\delta$,
	\beqan
	\|E_\tau\|^2 &\leq& \tau \sigma^2 + 2\sigma^2\sqrt{2\tau C_\tau(\delta)} + 2\sigma^2C_\tau(\delta)\\
	\|E_\tau\|^2 &\geq& \tau \sigma^2 - 2\sigma^2\sqrt{\tau C_\tau(\delta)}\,,
	\eeqan
	where $C_\tau(\delta)=\ln(e/\delta)(1+c_\tau/\ln(1/\delta))$.
	
	Thus, combining these two results with a union bound, we deduce that with probability higher than $1-3\delta$ it holds that
	\beqan
	\hat \sigma^2_{\lambda,\tau} &\leq& \sigma^2 + 2\sigma^2\sqrt{\frac{2C_\tau(\delta)}{\tau}}
	+ \frac{2\sigma^2C_\tau(\delta)}{\tau} \\
	&&+ \frac{\lambda^2}{\tau \lambda_{\min(G_{\lambda,\tau})}} \|\theta^\star\|^2_2\Big(1 - \frac{\lambda}{\lambda_{\max}(G_{\lambda,\tau})}\Big)
	- 2\frac{\sigma\lambda^2}{\tau \lambda_{\min^{3/2}(G_{\lambda,\tau})}}\|\theta^\star\|_2\sqrt{D_{\lambda_\star,\tau}(\delta)}\\
	\hat \sigma^2_{\lambda,\tau}  &\geq& \sigma^2  - 2\sigma^2\sqrt{\frac{C_\tau(\delta)}{\tau}}
	+ 
	\frac{\lambda^2}{\tau \lambda_{\max}(G_{\lambda,\tau})} \|\theta^\star\|^2_2\Big(1 - \frac{\lambda}{\lambda_{\min(G_{\lambda,\tau})}}\Big)\\
	&&
	- 2\frac{\lambda^2\sigma}{\tau \lambda_{\min^{3/2}(G_{\lambda,\tau})}}\|\theta^\star\|_2\sqrt{D_{\lambda_\star,\tau}(\delta)}
	-\frac{\sigma^2D_{\lambda_\star,\tau}(\delta)}{\tau}\Big(1 + \frac{\lambda}{\lambda_{\min(G_{\lambda,\tau})}}\Big)
	\,.
	\eeqan  
	We can now derive a bound on $\sqrt{\hat \sigma^2_{\lambda,\tau}}$. Indeed, 
	\beqan
	\hat \sigma^2_{\lambda,\tau} &\leq& \bigg(
	\sigma + \sqrt{\frac{2\sigma^2C_\tau(\delta)}{\tau}}\bigg)^2+ \frac{\lambda^2}{\tau \lambda_{\min(G_{\lambda,\tau})}} \|\theta^\star\|^2_2\Big(1 - \frac{\lambda}{\lambda_{\max}(G_{\lambda,\tau})}\Big)\\  
	\hat \sigma^2_{\lambda,\tau} &\geq& \bigg(
	\sigma - \sqrt{\frac{\sigma^2C_\tau(\delta)}{\tau}}\bigg)^2 - \frac{\sigma^2}{\tau}
	\bigg(C_\tau(\delta) +D_{\lambda_\star,\tau}(\delta)\Big(1 + \frac{\lambda}{\lambda_{\min(G_{\lambda,\tau})}}\Big)\bigg)\\
	&& - \frac{2\lambda^2\sigma}{\tau \lambda_{\min^{3/2}(G_{\lambda,\tau})}}\|\theta^\star\|_2\sqrt{D_{\lambda_\star,\tau}(\delta)}\,.
	\eeqan
Thus, using the inequality $\sqrt{a+b} \leq \sqrt{a} + \sqrt{b}$, on both inequalities, we get
	\beqan
	\sqrt{\hat \sigma^2_{\lambda,\tau}} &\leq& 
	\sigma + \sigma\sqrt{\frac{2C_\tau(\delta)}{\tau}}     
	+ \frac{\lambda\|\theta^\star\|_2}{\sqrt{\tau \lambda_{\min(G_{\lambda,\tau})}}} \sqrt{1 - \frac{\lambda}{\lambda_{\max}(G_{\lambda,\tau})}}
	\\   
	\sqrt{\hat \sigma^2_{\lambda,\tau}} &\geq& 
	\sigma - \sigma\sqrt{\frac{C_\tau(\delta)}{\tau}}-
	\sigma\sqrt{\frac{C_\tau(\delta) + D_{\lambda_\star,\tau}(\delta)\Big(1\! +\! \frac{\lambda}{\lambda_{\min(G_{\lambda,\tau})}}\Big)}{\tau}}\\
	&&-\lambda\sqrt{\frac{2\sigma \|\theta^\star\|_2\sqrt{D_{\lambda_\star,\tau}(\delta)}}{\tau \lambda_{\min^{3/2}(G_{\lambda,\tau})}}}\,.
	\eeqan
\end{proof}


\begin{corollary}[Extension of Corollary~3.13 in \cite{Maillard2016}]
\label{cor:rkhs_var_bound}
    With probability higher than $1 - 3 \delta'$, it holds simultaneously over all $t\geq0$,
    \begin{align*}
        \sigma
        & \leq \frac{1}{\alpha^2} \Bigg( \sqrt{\frac{\sqrt{\lambda} \lVert  f_\star \rVert_\cK \sqrt{D_{t,\lambda_\star}(\delta')}}{2t}} + \sqrt{\frac{\sqrt{\lambda} \lVert  f_\star \rVert_\cK \sqrt{D_{\lambda_\star,t}(\delta')}}{2t} + \hat \sigma_{\lambda,t} \alpha} \Bigg)^2 \\
        \sigma
        & \geq \bigg[ \hat \sigma_{\lambda,t} - \lVert f_\star \rVert_\cK \sqrt{\frac{\lambda}{t} \bigg(1 - \frac{1}{\max_{t' \leq t}(1 + k_{\lambda,t'-1}(x_{t'}, x_{t'}))} \bigg)} \bigg] \bigg( 1 + \sqrt{\frac{2 C_t(\delta')}{t}} \bigg)^{-1},
    \end{align*}
    where $\alpha = \max \bigg( 1 - \sqrt{\frac{C_t(\delta')}{t}} - \sqrt{\frac{C_t(\delta') + 2 D_{\lambda_\star,t}(\delta')}{t}} , 0 \bigg)$. Further, if an upper bound $\sigma^+ \geq \sigma$ is known, one can derive the following inequalities that hold with probability higher than $1 - 3 \delta'$,
    \begin{align*}
        \sigma
        & \leq \hat \sigma_{\lambda,t} + \sigma^+ \bigg( \sqrt{\frac{C_t(\delta')}{t}} + \sqrt{\frac{C_t(\delta')+2D_{\lambda_\star,t}(\delta')}{t}} \bigg) + \sqrt{\frac{2 \sigma^+ \lambda^{1/2} \lVert f_\star \rVert_\cK \sqrt{D_{t,\lambda_\star}(\delta')}}{t}} \\
        \sigma
        & \geq \hat \sigma_{\lambda,t} - \sigma^+ \sqrt{\frac{2C_t(\delta')}{t}} - \lVert f_\star \rVert_\cK \sqrt{\frac{\lambda}{t} \bigg( 1 - \frac{1}{\max_{t' \leq t} (1 + k_{\lambda,t'-1}(x_{t'}, x_{t'}))} \bigg)}.
    \end{align*}
\end{corollary}

\begin{proof}
    Using Theorem~\ref{thm:kLSvar}, it holds with high probability that
    \begin{align*}
        \underbrace{\hat \sigma_{\lambda,\tau}}_A \geq \sigma \bigg[ \underbrace{1 - \sqrt{\frac{C_\tau(\delta')}{\tau}} - \sqrt{\frac{C_\tau(\delta') + 2 D_{\lambda_\star,\tau}(\delta')}{\tau}}}_C \bigg] - \sqrt{\sigma} \underbrace{\sqrt{\frac{2 \sqrt{\lambda} \lVert  f_\star \rVert_\cK \sqrt{D_{\lambda_\star,\tau}(\delta')}}{\tau}}}_B.
    \end{align*}
    The inequality rewrites $A \geq \sigma C - \sqrt{\sigma} B$. Now, let $y^2 = \sigma$. If $C>0$,    the inequality holds provided that $y \geq 0$ and $A + yB - Cy^2 \geq 0$, that is when $0 \leq y \leq \frac{B + \sqrt{B^2 + 4AC}}{2C}$. We conclude  by choosing the stopping time $\tau$ corresponding to the probability of \emph{bad events}, as in the proof of Theorem~\ref{thm:kLSmeantuned}, then by remarking that $t\mapsto C_t(\delta')$ is an increasing function.
\end{proof}

\begin{mylemma}[Lemma~5.10 from \cite{Maillard2016}]\label{lem:varconcentration}
	Assume that $T_n$ is a random stopping time
	that satisfies $T_n\leq n$ almost surely, then
	\beqan
	\Pr\bigg[ \frac{1}{T_n}\sum_{i=1}^{T_n} \xi_i^2 \geq 
	\sigma^2 + 2\sigma^2\sqrt{\frac{2\ln(e/\delta)}{T_n}}
	+2\sigma^2\frac{\ln(e/\delta)}{T_n}\bigg] \leq 
	\Big( \lceil \ln(n)\ln(e/\delta)\rceil\Big)\delta\,,
	\eeqan
	\beqan
	\Pr\bigg[ \frac{1}{T_n}\sum_{i=1}^{T_n} \xi_i^2 \leq 
	\sigma^2 - 2\sigma^2\sqrt{\frac{\ln(e/\delta)}{T_n}}\bigg] \leq 
	\Big( \lceil \ln(n)\ln(e/\delta)\rceil\Big)\delta\,.
	\eeqan
	Further, for a random stopping time $T$, and if we introduce 
	$c_T =\ln(\pi^2\ln^2(T)/6)$, then
	\beqan
	\Pr\bigg[ \frac{1}{T}\sum_{i=1}^{T} \xi_i^2 \!\geq 
	\sigma^2 \!+\! 2\sigma^2\sqrt{\frac{2\ln(e/\delta)(1+c_T/\ln(1/\delta))}{T}}
	+2\sigma^2\frac{\ln(e/\delta)(1+c_T/\ln(1/\delta))}{T}\bigg] \leq \delta\,,
	\eeqan
	\beqan
	\Pr\bigg[ \frac{1}{T}\sum_{i=1}^{T} \xi_i^2 \leq 
	\sigma^2 - 2\sigma^2\sqrt{\frac{\ln(e/\delta)(1+c_T/\ln(1/\delta))}{T}}\bigg] \leq \delta\,.
	\eeqan
\end{mylemma}

\section{Application to stochastic multi-armed bandits}
\label{app:bandits}

\begin{proofof}{Lemma~\ref{lem:infogain}}
    Using the facts that $\min \{r, \alpha\} \leq (\alpha / \ln(1+\alpha)) \ln(1+r)$ and $\min_{\lambda \in \bslambda} \lambda \geq \sigma^2 / C^2$:
	\begin{align*}
	\sum_{t=1}^T s_{\bslambda, t-1}^2(x_t)
	& = \sigma^2 \sum_{t=1}^T \frac{1}{\lambda_t} k_{\lambda_t, t-1}(x_t, x_t) \\
	& \leq \sigma^2 \sum_{t=1}^T \frac{C^2}{\sigma^2} k_{\sigma^2/C^2, t-1}(x_t, x_t) \\
	& = \sigma^2 \sum_{t=1}^T \min \Big\{ \frac{C^2}{\sigma^2} k_{\sigma^2/C^2, t-1}(x_t, x_t), \frac{C^2}{\sigma^2} \Big\} \\
	& \leq \frac{2 C^2}{ \ln(1 + C^2/\sigma^2)} \gamma_T(\sigma^2/C^2).
	\end{align*}
    
    In particular, we obtain by a Cauchy-Schwarz inequality,
    \beqan
     \sum_{t=1}^T \sqrt{\frac{ k_{\lambda_t, t-1}(x_t, x_t)}{\lambda_t}}
     \leq      
    \sqrt{T \frac{2 C^2/\sigma^2}{\ln(1 + C^2/\sigma^2)} \gamma_T(\sigma^2/C^2)}\,.
    \eeqan
\end{proofof}

\begin{proofof}{Lemma~\ref{lem:Bt}}
	We want to control the quantity $B_{\lambda_t,t}(\delta)$.
	First of all, recall from Equation~\ref{eqn:Bt} that	
	\begin{align*}
	B_{\lambda_t,t}(\delta) &= \!\sqrt{\!\lambda_t}C \!+\! \sigma_{+,t} \sqrt{2\ln(1/\delta) + 2 \gamma_t(\lambda_-)} \\
	&\leq \sigma_+ + \sigma_+ \sqrt{2\ln(1/\delta) + 2 \gamma_t(\sigma_{t,-}^2/C^2)}\,,
	\end{align*}
	where we use the facts that $\lambda_t \leq \sigma_+^2/C^2$ and $\lambda_- \geq \sigma_{t,-}^2/C^2$.
	Then, using that  $\sigma_{t,-}^2\geq \sigma_{-}$, that $\gamma_t(\cdot)$ is non-increasing and non-decreasing with $t$, it comes
	\begin{align*}
	B_{\lambda_t,t}(\delta) 	&\leq \sigma_+ + \sigma_+ \sqrt{2\ln(1/\delta) + 2 \gamma_T(\sigma_{-}^2/C^2)}\,.
	\end{align*}

	Alternatively one may use Theorem~\ref{thm:kLSvar} in order to control the random variables 
	$\sigma_{t,+}$ and $\sigma_{t,-}$ in a tighter way.
	For instance, by Theorem~\ref{thm:kLSvar}, we easily obtain that with high probability, for all $t$,
	\beqan
\sigma\geq 	\sigma_{t,-} &\geq& \sigma - \frac{\sigma}{\sqrt{t}}
\frac{\bigg[(\sqrt{2}+1)\sqrt{C_t(\delta)}-
\sqrt{C_t(\delta) + 2D_{\lambda_\star,t}(\delta)}\bigg]}
{1 + \sqrt{2C_t(\delta)/t}}\\
&&-\frac{\sqrt{2 \sigma \lambda^{1/2}\|f^\star\|_\cK\sqrt{D_{\lambda_\star,t}(\delta)}} + C \sqrt{\lambda}\sqrt{1 - 
		\frac{1}{\max_{t\leq t}(1+k_{\lambda,t-1}(x_t,x_t))}}}{\sqrt{t} (1+\sqrt{2C_t(\delta)/t})}\,,
	\eeqan
	that is the estimate satisfies $\quad 	\sigma\geq 	\sigma_{t,-} \geq \sigma - O(1/\sqrt{t}).\quad$ 
	This in turns implies that 
	$\gamma_t(\sigma_{-,t}^2/C^2) \leq \gamma_t(\sigma^2/C^2) + O(1/\sqrt{t})$. Likewise, it can be shown that
	$\quad 	\sigma\leq 	\sigma_{t,+} \leq \sigma + O(1/\sqrt{t}),$  which yields
\begin{align*}
B_{\lambda_t,t}(\delta) 	&\leq \sigma\Big(1 + \sqrt{2\ln(1/\delta) + 2 \gamma_T(\sigma_{-}^2/C^2)}\Big)+o(1)\,.
\end{align*}	
	
\end{proofof}

\begin{proofof}{Theorem~\ref{thm:KernelUCB} (UCB algorithm for kernel bandits)}
Let $r_t$ denote the instantaneous regret	at time $t$ and $f^+(x_t)$ denote the optimistic value at the chosen point $x_t$, built from the confidence set used by the UCB algorithm. The following holds with probability higher than $1-4\delta$ for each time-step $t$
\beqan
r_t(\lambda_t) &=& f_\star(x_\star)-f_\star(x_t) \leq f^+_{t-1}(x_t)-f_\star(x_t)\\
&\leq& |f^+_{t-1}(x_t)-f_{\lambda_t,t-1}(x_t)|  + |f_{\lambda_t,t-1}(x_t)-f_\star(x_t)|\\
&\leq& 2 \sqrt{\frac{k_{\lambda_t,t-1}(x_t,x_t)}{\lambda_t}} B_{\lambda_t,t-1}(\delta)\,.
\eeqan
Thus, we deduce that with probability higher than $1-4\delta$:
\beqan
\kR_T &=& \sum_{t=1}^T r_t(\lambda)  \leq 2\sum_{t=1}^T\sqrt{\frac{k_{\lambda_t,t-1}(x_t,x_t)}{\lambda_t}} B_{\lambda_t,t-1}(\delta)\,.
\eeqan
We then use Lemma~\ref{lem:Bt} in order to control the term
$B_{\lambda_t,t-1}(\delta)$, and Lemma~\ref{lem:infogain} in order to control the sum of $\frac{k_{\lambda_t,t-1}(x_t,x_t)}{\lambda_t}$.
This yields the following bound on the regret: 
\beqan
\kR_T &\leq& 2\sigma_+\Big(1 + \sqrt{2\ln(1/\delta) + 2 \gamma_T(\sigma_{-}^2/C^2)}\big)
\sqrt{T \frac{2 C^2/\sigma^2}{\ln(1 + C^2/\sigma^2)} \gamma_T(\sigma^2/C^2)}\,.
\eeqan
\end{proofof}

\begin{proofof}{Theorem~\ref{thm:kernel_ts} (TS algorithm for kernel bandits)}	
	We closely follow the proof technique of \cite{Agrawal2014}, while clarifying and simplifying some steps. The general idea is to split the arms into two groups: \emph{saturated arms} and \emph{unsaturated arms}. The former designates arms where samples $\tilde f_t$ have low probability of dominating $f_\star(\star)$ while the latter designates the other case. This is related to the \emph{optimism}~\citep{Abeille2016}, that is the possibility of sampling a value that is higher than the optimum. 	
Let $\hat E_t$ and $\tilde E_t$ be the events that $\hat f_{t}$ and $\tilde f_t$ are concentrated around their respective means. More precisely, for a given confidence level $\delta$,  we introduce
	\beqan
	\hat E_{t,\delta} &=& \{\forall x\in\cX , |f_\star(x) - f_{\lambda_t,t-1}(x)| \leq  \hat C_{t,\delta}(x)  \}\\
	\tilde E_{t,\delta} &=&  \{\forall x\in\cX,  |f_{\lambda_t,t-1}(x) - \tilde f_t(x)|\leq  \tilde C_{t,\delta}(x)   \}\,,
	\eeqan
	for some quantities $\hat C_{t,\delta}(x) , \tilde C_{t,\delta}(x) $ to be defined.

\paragraph{Controlling the event $\hat E_{t,\delta}$}
Choosing the confidence bound to be
\beqan
\hat C_{t,\delta}(x)  = \sqrt{\frac{k_{\lambda_t,t-1}(x,x)}{\lambda_t}}B_{\lambda_t,t-1}(\delta/4)\,,
\eeqan
then the event $\hat E_{t,\delta}$ is controlled as  $\Pr\Big( \forall t \geq 0,\ \hat E_{t,\delta} \Big)\geq 1-\delta$.

\paragraph{Controlling the event $\tilde E_{t,\delta}$}
On the other hand, since 
$\tilde f_t(x)|\cH_{t-1} = \cN(f_{\lambda_t,t-1}(x), {\bf V}_t  )$ where we introduced the notation $ {\bf V}_t = v_t^2 \frac{\sigma_{+,t-1}^2}{\lambda_t} (k_{\lambda_t,t-1}(x,x'))_{x,x'\in\bX}$, then we have by a simple union bound over $x\in\bX$, 
\beqan
\Pr(	\tilde E_{t,\delta}^c |\cH_{t-1}  ) \leq \sum_{x\in\bX}\frac{1}{\sqrt{\pi} z_x }e^{-z_x^2/2}
\eeqan
provided that $z_x = \frac{\tilde C_{t,\delta}(x)}
{v_t \sqrt{\frac{\sigma_{+,t-1}^2}{\lambda_t}k_{\lambda_t,t-1}(x,x)}} \geq 1$ for all $x\in\bX$.
This motivates the following definition,
\beqan
\tilde C_{t,\delta}(x) = c_{t,\delta} v_t \sqrt{\frac{\sigma_{+,t-1}^2}{\lambda_t}k_{\lambda_t,t-1}(x,x)}\,,
\eeqan
for a well-chosen sequence $(c_{t,\delta})_t$. The choice 
$c_{t,\delta} = \max\{\sqrt{2\ln(t(t+1) |\bX|/\sqrt{\pi}\delta)},1\}$ ensures that
\beqan
\Pr( \exists t\geq 0\	\tilde E_{t,\delta}^c |\cH_{t-1}  )  &\leq& \sum_{t\geq 0}\frac{|\bX|}{\sqrt{\pi} c_{t,\delta} }e^{-c_{t,\delta}^2/2}=\sum_{t\geq 0}\frac{\delta}{c_{t,\delta} t(t+1) } \\
&\leq& \sum_{t\geq 0}\frac{\delta}{t(t+1) }  = \delta,
\eeqan
from which we obtain $\Pr\Big( \forall t \geq 0,\ \tilde E_{t,\delta} \Big)\geq 1-\delta$.

\paragraph{Summary}
By definition of the events, under $\hat E_{t,\delta}$ and $\tilde E_{t,\delta}$,
it thus holds that
\beqan
\forall x\in\cX,\
\Big|f_\star(x)- \tilde f_t(x)\Big| &\leq& \Big|f_\star(x) - f_{\lambda_t,t-1}(x)\Big| + \Big|f_{\lambda_t,t-1}(x) - \tilde f_t(x)\Big|\\
&\leq& \hat C_{t,\delta}(x) +\tilde C_{t,\delta}(x)\\
&=&\sqrt{\frac{k_{\lambda_t,t-1}(x,x)}{\lambda_t}}\bigg(B_{\lambda_t,t-1}(\delta/4) 
+ c_{t,\delta} v_t\sigma_{+,t-1}\bigg)\\
&=&s_{\blambda,t-1}(x)\bigg(\underbrace{\frac{B_{\lambda_t,t-1}(\delta/4)}{\sigma} 
	+ c_{t,\delta} v_t\frac{\sigma_{+,t-1}}{\sigma}}_{g_t(\delta)}\bigg)\,.
\eeqan

\paragraph{Saturated arms}
It is now convenient to introduce the set of saturated times a time $t$
\beqan
\cS_{t,\delta} = \bigg\{ x \in\bX :  
f_\star(\star)- f_\star(x) > s_{\blambda,t-1}(x)g_t(\delta) \bigg\}\quad\text{together with}\quad
x_{\cS,t} = \argmin_{x \notin \cS_{t,\delta}} s_{\blambda,t-1}(x)\,.
\eeqan
We remark that by construction $\star \notin \cS_{t,\delta}$ for all $t$. 
Now, by the strategy of the Kernel TS algorithm, $x_t = \argmax_{x\in\bX} \tilde f_t(x)$.
 Thus, we deduce that on the event $\hat E_{t,\delta}\cap \tilde E_{t,\delta}$
 \beqan
f_\star(\star)- f_\star(x_t) &=&
f_\star(\star)- f_\star(x_{\cS,t})  +f_\star(x_{\cS,t}) -f_\star(x_t)\\
&\leq&s_{\blambda,t-1}(x_{\cS,t})g_t(\delta)
+ \Big(f_\star(x_{\cS,t})- \tilde f_t(x_{\cS,t})\Big)\\
&& + \Big(\underbrace{\tilde f_t(x_{\cS,t}) - \tilde f_t(x_{t})}_{\leq 0}\Big) 
+\Big(\tilde f_t(x_{t}) -f_\star(x_t)\Big)\\
&\leq&2s_{\blambda,t-1}(x_{\cS,t})g_t(\delta) + 
s_{\blambda,t-1}(x_{t})g_t(\delta)\,.
 \eeqan
 Also, $f_\star(\star)- f_\star(x_t)\leq R$, where $R = \max_{x\in\bX}
 f_\star(\star)-f_\star(x)<\infty$.
We then remark that  by definition of $x_{\cS,t}$, we have
\beqan
\Esp[ s_{\blambda,t-1}(x_t) |\cH_{t-1}] &\geq &
\Esp[ s_{\blambda,t-1}(x_t)\indic{x_t \notin \cS_{t,\delta}} |\cH_{t-1}]\\
&\geq&
\Esp[ s_{\blambda,t-1}(x_{\cS,t})\indic{x_t \notin \cS_{t,\delta}} |\cH_{t-1}]\\
&=&s_{\blambda,t-1}(x_{\cS,t}) \Pr\bigg(x_t \notin \cS_{t,\delta} \bigg|\cH_{t-1}\bigg)\,.
\eeqan 
Likewise, 
\beqan
\min\{s_{\blambda,t-1}(x_{\cS,t})g_t(\delta),R \} \leq 
\frac{\Esp[ \min\{2s_{\blambda,t-1}(x_t)g_t(\delta),R\} |\cH_{t-1}]} {\Pr\bigg(x_t \notin \cS_{t,\delta} \bigg|\cH_{t-1}\bigg)}\,.
\eeqan
Since on the other hand,
$(f_\star(\star)- f_\star(x_t))\indic{x_t \notin \cS_{t,\delta}} \leq s_{\blambda,t-1}(x_t)g_t(\delta)\indic{x_t \notin \cS_{t,\delta}}$, we deduce that on the event $\hat E_{t,\delta}\cap \tilde E_{t,\delta}$ we have
\beqan
f_\star(\star)- f_\star(x_t) &\leq&
\min\bigg\{2s_{\blambda,t-1}(x_{\cS,t})g_t(\delta) + 
s_{\blambda,t-1}(x_{t})g_t(\delta),R\bigg\}\indic{x_t \in \cS_{t,\delta}}\\
&&+ s_{\blambda,t-1}(x_t)g_t(\delta)\indic{x_t \not\in \cS_{t,\delta}}\\
&\leq&\min\Big\{2s_{\blambda,t-1}(x_{\cS,t})g_t(\delta),R\Big\}\indic{x_t \in \cS_{t,\delta}} + 
s_{\blambda,t-1}(x_{t})g_t(\delta)\\
&\leq& \frac{\Esp[  \min\{2s_{\blambda,t-1}(x_t)g_t(\delta),R\} |\cH_{t-1}]}{ \Pr\bigg(x_t \notin \cS_{t,\delta} \bigg|\cH_{t-1}\bigg)}\indic{x_t \in \cS_{t,\delta}} +
s_{\blambda,t-1}(x_{t})g_t(\delta)\,.
\eeqan

\paragraph{Lower bounding the denominator}
At this point, we note that on the event $\hat E_{t,\delta} \cap \tilde E_{t,\delta}$, for all $x\in\cS_{t,\delta}$,
\beqan
\tilde f_t(x) \leq f_\star(x) + s_{\blambda,t-1}(x)g_t(\delta) \leq f_\star(\star)\,,
\eeqan
while on the other hand  we have the inclusion $\quad\{\forall x\in\cS_{t,\delta},\  \tilde f_t(\star)> \tilde f_t(x) \} \subset 	\{ x_t \not \in\cS_{t,\delta} \}.\quad$ 
Thus,  combining these two properties, we deduce that
\beqan
\lefteqn{
	\{ x_t \in\cS_{t,\delta} \} \cap \hat E_{t,\delta}  \cap \tilde E_{t,\delta}}\\
&\subset &
\Big\{\exists x\in\cS_{t,\delta},  \tilde f_t(\star) \leq  \tilde f_t(x) \Big\}  \cap 
\Big\{ \forall x\in\cS_{t,\delta}, \tilde f_t(x) \leq f_\star(\star)\Big\}\\
&	\subset &	\Big\{\tilde f_t(\star) \leq  f_\star(\star)  \Big\}\,.
\eeqan
Further, using that  $\tilde f_t(x)|\cH_{t-1} = \cN(f_{\lambda_t,t-1}(x), {\bf V}_t  )$ yields
\beqan
\lefteqn{
	\{ x_t \in\cS_{t,\delta} \} \cap \hat E_{t,\delta}  \cap \tilde E_{t,\delta}}\\
&\subset &	\Big\{\tilde f_t(\star) - f_{\lambda_t,t-1}(\star) \leq  f_\star(\star)-f_{\lambda_t,t-1}(\star) \Big\}\cap \hat E_{t,\delta}  \cap \tilde E_{t,\delta}\\
&\subset&\Big\{\tilde f_t(\star) - f_{\lambda_t,t-1}(\star) \leq  \hat C_{t,\delta}(\star)\Big\}\subset
\Big\{\big|\tilde f_t(\star) - f_{\lambda_t,t-1}(\star)\big| \leq  \hat C_{t,\delta}(\star)\Big\}\,,
\eeqan
from which we obtain 
\beqan
\Big\{\big|\tilde f_t(\star) - f_{\lambda_t,t-1}(\star)\big|>  \hat C_{t,\delta}(\star)  \Big\} \cap \hat E_{t,\delta} \subset 	\{ x_t \not\in\cS_{t,\delta} \} \cup \tilde E_{t,\delta}^c\,.
\eeqan
Thus, we have proved that
\beqan
 \Pr\bigg(x_t \notin \cS_{t,\delta} \bigg|\cH_{t-1}\bigg)& \geq &
  \Pr\bigg(\big|\tilde f_t(\star) - f_{\lambda_t,t-1}(\star)\big|>  \hat C_{t,\delta}(\star) , \hat E_{t,\delta} \bigg|\cH_{t-1}\bigg)  - \Pr\Big(\tilde E_{t,\delta}^c\Big|\cH_{t-1}\Big)\\\
  &=& \Pr\bigg(\big|\tilde f_t(\star) - f_{\lambda_t,t-1}(\star)\big|>  \hat C_{t,\delta}(\star) \bigg|\cH_{t-1}\bigg)\indic{ \hat E_{t,\delta} }  - \Pr\Big(\tilde E_{t,\delta}^c\Big|\cH_{t-1}\Big)\,.
\eeqan

\paragraph{Anti-concentration}
We now resort to an anti-concentration result for Gaussian variables~\citep{Abramowitz1964}. More precisely, the following inequality holds
\beqan
\Pr\bigg( \bigg|\tilde f_t(\star) - f_{\lambda_t,t-1}(\star) \bigg| >   \hat C_{t,\delta}(\star) \bigg|\cH_{t-1}\bigg) \geq 
\frac{1}{2\sqrt{\pi}z}e^{-z^2/2}
\eeqan
where  we introduced the $\cH_{t-1}$-measurable random variable
\beqan
z &=& \frac{ \hat C_{t,\delta}(\star)}{v_t\sigma_{+,t-1}\sqrt{\frac{k_{\lambda_t,t-1}(\star,\star)}{\lambda_t}}}= \frac{ B_{\lambda_t,t-1}(\delta/4)}{v_t\sigma_{+,t-1}}\,,\quad
\text{	provided that}\ z\geq 1\,.
\eeqan
Taking $v_t = \frac{ B_{\lambda_t,t-1}(\delta/4)}{\sigma_{+,t-1}\sqrt{2\alpha_t \ln(\beta_t)}}$ for constants $\alpha_t,\beta_t$ such that $2\alpha_t \ln(\beta_t)\geq 1$ thus yields 
\beqan
\Pr\bigg(\Big|\tilde f_t(\star) - f_{\lambda_t,t-1}(\star)\Big| >   \hat C_{t,\delta}(\star) \bigg|\cH_{t-1}\bigg) \geq p_t \eqdef
\frac{\beta_t^{-\alpha_t}}{2\sqrt{\pi}\sqrt{2\alpha_t \ln(\beta_t)}}\,.
\eeqan

\paragraph{Summary}
At this point of the proof, we have proved that
\beqan
\lefteqn{
(f_\star(\star)- f_\star(x_t))\indic{\hat E_{t,\delta} \cap \tilde E_{t,\delta}}}\\
&\leq& \frac{\Esp[ \min\{2s_{\blambda,t-1}(x_t)g_t(\delta),R\} |\cH_{t-1}]\indic{x_t \in \cS_{t,\delta}}}{ 
p_t\indic{\hat E_{t,\delta}} - \Pr(\tilde E_{t,\delta}^c \ |\cH_{t-1})
}\indic{\hat E_{t,\delta} \cap \tilde E_{t,\delta}} + 
s_{\blambda,t-1}(x_{t})g_t(\delta)\indic{\hat E_{t,\delta} \cap \tilde E_{t,\delta}}\\
&\leq& 
\Esp[  \min\{2s_{\blambda,t-1}(x_t)g_t(\delta),R\} |\cH_{t-1}]
\bigg(\frac{1}{p_t}  + \frac{ \Pr(\tilde E_{t,\delta}^c|\cH_{t-1})}{p_t^2}\bigg)+
s_{\blambda,t-1}(x_{t})g_t(\delta)\,,
\eeqan
where in the second inequality, we used the property $\frac{1}{p-q} = \frac{1}{p} + \frac{q}{p(p-q)} \leq \frac{1}{p} + \frac{q}{p^2}$, for $p>q$. Combining the bound on 
$\Pr(\tilde E_{t,\delta}^c|\cH_{t-1})$ and the definition of $p_t$, we obtain
\beqan
\lefteqn{
	(f_\star(\star)- f_\star(x_t))\indic{\hat E_{t,\delta} \cap \tilde E_{t,\delta}}}\\
&\leq&
\Esp[ \min\{2s_{\blambda,t-1}(x_t)g_t(\delta),R\} |\cH_{t-1}]
\bigg( \sqrt{8\pi\alpha_t \ln(\beta_t)}\beta_t^{\alpha_t} + 
\delta\frac{8\pi\alpha_t \ln(\beta_t)\beta_t^{2\alpha_t}}{c_{t,\delta}t(t+1)}
\bigg)+
s_{\blambda,t-1}(x_{t})g_t(\delta)\,. 
\eeqan

\paragraph{Pseudo-regret}
Summing-up the previous terms over $t\geq1$, we obtain that the pseudo-regret of the Kernel TS strategy satisfies,
on the event $\bigcap_{t\geq 1} \hat E_{t,\delta} \cap \tilde E_{t,\delta}$ that holds with probability higher than $1-2\delta$, 
\beqan
\kR_T \leq 
\sum_{t=1}^T \bigg[ \Esp[ \min\{2s_{\blambda,t-1}(x_t)g_t(\delta),R\} |\cH_{t-1}]
 \sqrt{8\pi\alpha_t \ln(\beta_t)}\beta_t^{\alpha_t}\bigg(1 + 
\delta\frac{\sqrt{8\pi\alpha_t \ln(\beta_t)}\beta_t^{\alpha_t}}{c_{t,\delta}t(t+1)}
\bigg)+
s_{\blambda,t-1}(x_{t})g_t(\delta) \bigg]\,,
\eeqan
where $c_{t,\delta}=\max\{\sqrt{2\ln(t(t+1) |\bX|/\sqrt{\pi}\delta)},1\}$, and the constants $\alpha_t,\beta_t$ must be such that  $2\alpha_t \ln(\beta_t)\geq 1$ and $\sqrt{8\pi\alpha_t \ln(\beta_t)}\beta_t^{\alpha_t}\geq 1$. Also, let us recall that
\beqan
g_t(\delta)&=&\frac{B_{\lambda_t,t-1}(\delta/4)}{\sigma} 
+ c_{t,\delta} v_t\frac{\sigma_{+,t-1}}{\sigma}\\
&=&\frac{B_{\lambda_t,t-1}(\delta/4)}{\sigma}\bigg(1 + \frac{c_{t,\delta}}{\sqrt{2\alpha_t\ln(\beta_t)}}\bigg)\,.
\eeqan
In particular, the specific choice $\alpha_t=1/2\ln(\beta_t)$ where $\beta_t>1$  (which satisfies $1\geq 1$ and $\sqrt{4\pi e}\geq 1$) yields
\beqan
\kR_T &\leq&
\sum_{t=1}^T \Esp\bigg[ \min\{2s_{\blambda,t-1}(x_t)\frac{B_{\lambda_t,t-1}(\delta/4)}{\sigma}\Big(1+
c_{t,\delta}\Big),R\} \bigg|\cH_{t-1}\bigg]
\eta_t
+ s_{\blambda,t-1}(x_{t})g_t(\delta)\\
&=&\sum_{t=1}^T \Esp\Big[ \min\big\{2s_{\blambda,t-1}(x_t)g_t(\delta),R\big\} \Big|\cH_{t-1}\Big]
\eta_t
+ s_{\blambda,t-1}(x_{t})g_t(\delta)\,,
\eeqan
where we introduced the deterministic quantity
$\quad\eta_t= \sqrt{4\pi e}\Big(1 + 
\delta\frac{\sqrt{4\pi e}}{c_{t,\delta}t(t+1)}\Big)\,.\quad$


\paragraph{Concentration}
In order to finish the proof, we now relate the sum of 
the terms $\Esp[ s_{\blambda,t-1}(x_t) |\cH_{t-1}]$, $t\geq1$ to the sum of the terms $s_{\blambda,t-1}(x_t)$. 
More precisely, let us introduce the following random variable
\beqan
X_t =  \Esp\Big[ \min\big\{2s_{\blambda,t-1}(x_t)g_t(\delta),R\big\} \Big|\cH_{t-1}\Big]\eta_t - 
 \min\big\{2s_{\blambda,t-1}(x_t)g_t(\delta),R\big\}\eta_t\,.
\eeqan
By construction, $\Esp[X_t | \cH_{t-1}] = 0$ and 
$|X_t| \leq R\eta_t\,.$  Thus, by an application of Azuma-hoeffding's inequality for martingales, we obtain that for all $\delta\in(0,1)$, with probability higher than $1-\delta$, 
\beqan
\sum_{t=1}^T X_t \leq \sqrt{ 2\sum_{t=1}^T R^2\eta_t^2 \ln(1/\delta)}\,,
\eeqan
and thus that on an event of probability higher than $1-3\delta$,
\beqan
\kR_T &\leq& \sum_{t=1}^T \min\big\{2s_{\blambda,t-1}(x_t)g_t(\delta),R\big\}\eta_t
+ s_{\blambda,t-1}(x_{t})g_t(\delta)+
\sqrt{ 2\sum_{t=1}^T R^2\eta_t^2 \ln(1/\delta)}\,.
\eeqan
Replacing $\eta_t$ with its expression, that is
\beqan
\eta_t&=& \sqrt{4\pi e}\Big(1 + 
\delta\frac{\sqrt{4\pi e}}{\max\{\sqrt{2\ln(t(t+1) |\bX|/\sqrt{\pi}\delta)},1\}t(t+1)}\Big)\\
&\leq&\sqrt{4\pi e}\Big(1 +\delta\frac{\sqrt{4\pi e}}{t(t+1)}\Big)\,,
\eeqan
we deduce that with probability higher than $1-3\delta$,
\beqan
\kR_T &\leq& (4\sqrt{\pi e}+1)\bigg(\sum_{t=1}^T s_{\blambda,t-1}(x_t)g_t(\delta)\bigg)
+ R \delta 4\pi e
+R\sqrt{ 8 \pi e\sum_{t=1}^T (1 + \delta\frac{\sqrt{4\pi e}}{t(t+1)})^2 \ln(1/\delta)}\\
&\leq&
 (4\sqrt{\pi e}+1)\bigg(\sum_{t=1}^T s_{\blambda,t-1}(x_t)g_t(\delta)\bigg)
+ R \delta 4\pi e
+\sqrt{ 8 \pi e(1 + \delta\sqrt{4\pi e})^2} R\sqrt{T \ln(1/\delta)}\\
&=&(4\sqrt{\pi e}+1)\bigg(\sum_{t=1}^T \sqrt{\frac{k_{\lambda_t,t-1}(x_t,x_t)}{\lambda_t}} B_{\lambda_t,t-1}(\delta/4)(1+c_{t,\delta})\bigg)\\
&&+ R \delta 4\pi e+\sqrt{ 8 \pi e(1 + \delta\sqrt{4\pi e})^2} R\sqrt{T \ln(1/\delta)}\,.
\eeqan
This concludes the proof of the main result, since
$c_{t,\delta}\leq c_{T,\delta}$.

\paragraph{Final bound}
Then, using Lemma~\ref{lem:Bt} we can rewrite the regret as
\beqan
\kR_T
&=&(4\sqrt{\pi e}+1)(1+c_{T,\delta})  \sigma_+\Big(1  + \sqrt{2\ln(4/\delta) + 2 \gamma_T(\sigma_-^2/C^2)} \Big) \sum_{t=1}^T \sqrt{\frac{k_{\lambda_t,t-1}(x_t,x_t)}{\lambda_t}} \\
&&+ R \delta 4\pi e+\sqrt{ 8 \pi e(1 + \delta\sqrt{4\pi e})^2} R\sqrt{T \ln(1/\delta)}\,.
\eeqan
Using Lemma~\ref{lem:infogain} together with a Cauchy-Schwarz inequality, we finally obtain
\beqan
\kR_T
&=&(4\sqrt{\pi e}+1)(1+c_{T,\delta})  \sigma_+\Big( 1 +\sqrt{2\ln(4/\delta) + 2 \gamma_T(\sigma_-^2/C^2)} \Big)  \sqrt{T \frac{2 C^2/\sigma^2}{\ln(1 + C^2/\sigma^2)} \gamma_T(\sigma^2/C^2)} \\
&&+ R \delta 4\pi e+\sqrt{ 8 \pi e(1 + \delta\sqrt{4\pi e})^2} R\sqrt{T \ln(1/\delta)}\,.
\eeqan
\end{proofof}

%
%


\end{document}